%% file: PPRL.tex
\documentclass{article}

\PassOptionsToPackage{numbers, compress}{natbib}

  \usepackage[preprint]{neurips_2019}

\usepackage[utf8]{inputenc} 
\usepackage[T1]{fontenc}    
\usepackage{url}            
\usepackage{booktabs}       
\usepackage{blindtext}
\usepackage{amsfonts}       
\usepackage{nicefrac}       
\usepackage{microtype}      
\usepackage{floatrow}
\usepackage{xcolor}
\usepackage{math} 
\usepackage[utf8]{inputenc} 
\usepackage{amsmath,amssymb,amsthm,mathrsfs,amsfonts,dsfont}
\usepackage[T1]{fontenc}    
\usepackage{hyperref}
\usepackage{geometry}
\usepackage{url}            
\usepackage{booktabs}       
\usepackage{amsfonts}       
\usepackage{nicefrac}       
\usepackage{microtype}      
\usepackage{times}
\usepackage{graphicx} 
\usepackage{subcaption}
\usepackage{natbib}
\usepackage{algorithm}
\usepackage{algorithmic}
\usepackage{tikz}
\usetikzlibrary{arrows,calc,shapes,intersections}
\usepackage{xcolor}
\usepackage{mathtools}
\usepackage{paralist}
\usepackage{multirow}
\usepackage{thmtools,thm-restate}
\newfloatcommand{capbtabbox}{table}[][0.28\textwidth]
\title{Policy Poisoning\\in Batch Reinforcement Learning and Control}

\author{%
  Yuzhe Ma \\
  University of Wisconsin--Madison\\
  \texttt{yzm234@cs.wisc.edu} \\
  \And
  Xuezhou Zhang \\
  University of Wisconsin--Madison \\
  \texttt{zhangxz1123@cs.wisc.edu} \\
  \AND
  Wen Sun \\
  Microsoft Research New York \\
  \texttt{Sun.Wen@microsoft.com} \\
  \And
  Xiaojin Zhu\\
  University of Wisconsin--Madison \\
  \texttt{jerryzhu@cs.wisc.edu} \\
}
\input{def}
\begin{document}

\maketitle

\begin{abstract}
We study a security threat to batch reinforcement learning and control where the attacker aims to poison the learned policy.
The victim is a reinforcement learner / controller which first estimates the dynamics and the rewards from a batch data set, and then solves for the optimal policy with respect to the estimates. 
The attacker can modify the data set slightly before learning happens,
and wants to force the learner into learning a target policy chosen by the attacker. 
We present a unified framework for solving batch policy poisoning attacks,
and instantiate the attack on two standard victims:
tabular certainty equivalence learner in reinforcement learning and linear quadratic regulator in control. 
We show that both instantiation result in a convex optimization problem on which global optimality is guaranteed, and provide analysis on attack feasibility and attack cost.
Experiments show the effectiveness of policy poisoning attacks.
\end{abstract}

\section{Introduction}
With the increasing adoption of machine learning, it is critical to study security threats to learning algorithms and design effective defense mechanisms against those threats.
There has been significant work on adversarial attacks~\cite{biggio2018wild,huang2011adversarial}.
We focus on the subarea of data poisoning attacks where the adversary manipulates the training data so that the learner learns a wrong model.
Prior work on data poisoning targeted victims in supervised learning~\cite{mei2015using,koh2018stronger,wang2018data,zhang2019online} and multi-armed bandits~\cite{jun2018adversarial,ma2018data,liu2019data}.
We take a step further and study data poisoning attacks on reinforcement learning (RL).
Given RL's prominent applications in robotics, games and so on, an intentionally and adversarially planted bad policy could be devastating.

While there has been some related work in test-time attack on RL, reward shaping, and teaching inverse reinforcement learning (IRL), little is understood on how to training-set poison a reinforcement learner. 
We take the first step and focus on \emph{batch} reinforcement learner and controller as the victims.  These victims learn their policy from a batch training set. 
We assume that the attacker can modify the rewards in the training set, which we show is sufficient for policy poisoning.
The attacker's goal is to force the victim to learn a particular target policy (hence the name policy poisoning), while minimizing the reward modifications.
Our main contribution is to characterize batch policy poisoning with a unified optimization framework, and to study two instances against tabular certainty-equivalence (TCE) victim and linear quadratic regulator (LQR) victim, respectively.

\input{previous}
\input{definition}

\section{Policy Poisoning} 
We study policy poisoning attacks on model-based batch RL learners.
Our threat model is as follows:

\textbf{Knowledge of the attacker.} 
The attacker has access to the original training set $D^0 = (s_t,a_t,r^0_t,s_t^\prime)_{t=0:T-1}$.
The attacker knows the model-based RL learner's algorithm.
Importantly, the attacker knows how the learner estimates the environment, i.e.,~\eqref{MLE} and~\eqref{MMSE}. In the case~\eqref{MLE} has multiple maximizers, we assume the attacker knows exactly the $\hat P$ that the learner picks.

\textbf{Available actions of the attacker.} 
The attacker is allowed to arbitrarily modify the rewards 
$\bv{r}^0=(r^0_0, ...,r^0_{T-1})$ 
in $D^0$ into $\bv{r}=(r_0,..., r_{T-1})$. 
As we show later, changing $r$'s but not $s,a,s'$ is sufficient for policy poisoning.

\textbf{Attacker's goals.} 
The attacker has a pre-specified target policy $\pi^\dagger$. 
The attack goals are to (1) force the learner to learn $\pi^\dagger$, (2) minimize attack cost $\|\bv r - \bv r^0\|_\alpha$ under an $\alpha$-norm chosen by the attacker.

Given the threat model, we  can formulate policy poisoning as a bi-level optimization problem\footnote{As we will show, the constraint~\eqref{eq:attack_goal} could lead to an open feasible set (e.g., in~\eqref{TCE:goal}) for the attack optimization~\eqref{eq:objective}-\eqref{eq:attack_goal}, on which the minimum of the objective function~\eqref{eq:objective} may not be well-defined. In the case~\eqref{eq:attack_goal} induces an open set, we will consider instead a closed subset of it, and optimize over the subset. How to construct the closed subset will be made clear for concrete learners  later.}:
\begin{eqnarray}
\min_{\bv r,\hat R} &&\|\bv r - \bv r^0\|_\alpha \label{eq:objective}\\
\mbox{s.t. } &&\hat R = \argmin_{R} \sum_{t=0}^{T-1} (r_t - R(s_t,a_t))^2 \label{eq:estimate_R} \\
&&\{\pi^\dagger\} = \argmax_{\pi : \S \mapsto \A} \mathbb{E}_{\hat{P}} \sum_{\tau=0}^\infty \gamma^\tau\hat R(s_\tau,\pi(s_\tau)).  \label{eq:attack_goal}
\end{eqnarray}
The $\hat P$ in~\eqref{eq:attack_goal} does not involve $\bv r$ and is precomputed from $D^0$.
The singleton set $\{\pi^\dagger\}$ on the LHS of~\eqref{eq:attack_goal} ensures that the target policy is learned uniquely, i.e., there are no other optimal policies tied with $\pi^\dagger$.
Next, we instantiate this attack formulation to two representative model-based RL victims.

\subsection{Poisoning a Tabular Certainty Equivalence (TCE) Victim}\label{sec:poison_TCE}
In tabular certainty equivalence (TCE), the environment is a Markov Decision Process (MDP) with finite state and action space.
Given original data $D^0=(s_t,a_t,r_t^0, s_t^\prime )_{0:T-1}$, let $T_{s,a}=\{t\mid s_t=s,a_t=a\}$, the time indexes of all training items for which action $a$ is taken at state $s$. We assume $T_{s,a}\ge 1$, $\forall s,a$, i.e., each state-action pair appears at least once in $D^0$. This condition is needed to ensure that the learner's estimate $\hat P$ and $\hat R$ exist. Remember that we require~\eqref{MMSE} to have a unique solution. For the TCE learner, $\hat R$ is unique as long as it exists. Therefore, $T_{s,a}\ge 1$, $\forall s,a$ is sufficient to guarantee a unique solution to~\eqref{MMSE}.
Let the poisoned data be $D=(s_t,a_t,r_t, s_t^\prime)_{0:T-1}$.
Instantiating model estimation~\eqref{MLE},~\eqref{MMSE} for TCE, we have
\begin{equation}
\label{eq:tabular_transition}
\hat P(s^\prime\mid s,a)=\frac{1}{|T_{s,a}|}\sum_{t\in T_{s,a}}\ind{s_t^\prime=s^\prime},
\end{equation}
where $\ind{}$ is the indicator function, and
\begin{equation}
\label{eq:tabular_reward}
\hat R(s,a)=\frac{1}{|T_{s,a}|}\sum_{t\in T_{s,a}}r_t.
\end{equation}
The TCE learner uses $\hat P, \hat R$ to form an estimated MDP $\hat M$, then solves for the optimal policy $\hat\pi$ with respect to $\hat M$ using the Bellman equation~\eqref{eq:Bellman}.
The attack goal \eqref{eq:attack_goal} can be naively characterized by
\begin{equation}\label{TCE:goal}
Q(s,\pi^\dagger(s))>Q(s,a), \forall s\in \S, \forall a\neq \pi^\dagger(s).
\end{equation}
However, due to the strict inequality,~\eqref{TCE:goal} induces an open set in the $Q$ space, on which the minimum of~\eqref{eq:objective} may not be well-defined. 
Instead, we require a stronger attack goal which leads to a closed subset in the $Q$ space. This is defined as the following $\epsilon$-robust target $Q$ polytope.
\begin{definition}{($\epsilon$-robust target $Q$ polytope)} The set of $\epsilon$-robust $Q$ functions induced by a target policy $\pi^\dagger$ is the polytope 
\begin{equation}\label{TCE:polytope}
\mathcal Q_\epsilon(\pi^\dagger)=\{Q: Q(s,\pi^\dagger(s))\ge Q(s,a)+\epsilon, \forall s\in  \S,\forall a\neq \pi^\dagger(s)\}
\end{equation}
for a fixed $\epsilon>0$.
\end{definition}
The margin parameter $\epsilon$ ensures that $\pi^\dagger$ is the unique optimal policy for any $Q$ in the polytope.
We now have a solvable attack problem, where the attacker wants to force the victim's $Q$ function into the $\epsilon$-robust target $Q$ polytope $\mathcal Q_\epsilon(\pi^\dagger)$:
\begin{eqnarray}\label{attack:formulation}
  \min_{\bv{r} \in \mathbb{R}^T, \hat R, Q \in \mathbb{R}^{|\S| \times|\A|}}  &&\|\bv r-\bv r^0\|_\alpha\label{attack:objective}\\ 
\mbox{s.t. } && \hat R(s,a)=\frac{1}{|T_{s,a}|}\sum_{t\in T_{s,a}}r_t \label{attack:c1}\\
             &&  Q(s, a)=\hat R(s, a)+\gamma \sum_{s^{\prime}} \hat P\left(s^{\prime} | s, a\right) Q\left(s^{\prime}, \pi^\dagger(s^\prime)\right), \forall s, \forall a\label{attack:c2}\\
  &&Q(s,\pi^\dagger(s))\ge Q(s,a)+\epsilon, \forall s\in  \S,\forall a\neq \pi^\dagger(s).\label{attack:c3}
\end{eqnarray}
The constraint~\eqref{attack:c2} enforces Bellman optimality on the value function $Q$, in which $\max_{a'\in\A}Q(s',a')$ is replaced by $Q\left(s^{\prime}, \pi^\dagger(s^\prime)\right)$, since the target policy is guaranteed to be optimal by \eqref{attack:c3}. 
Note that problem \eqref{attack:objective}-\eqref{attack:c3}  is a convex program with linear constraints given that $\alpha\geq 1$, thus could be solved to global optimality.
However, we point out that~\eqref{attack:formulation}-\eqref{attack:c3} is a more stringent formulation than~\eqref{eq:objective}-\eqref{eq:attack_goal} due to the additional margin parameter $\epsilon$ we introduced. 
The feasible set of~\eqref{attack:formulation}-\eqref{attack:c3} is a subset of~\eqref{eq:objective}-\eqref{eq:attack_goal}. 
Therefore, the optimal solution to~\eqref{attack:formulation}-\eqref{attack:c3} could in general be a sub-optimal solution to~\eqref{eq:objective}-\eqref{eq:attack_goal} with potentially larger objective value.
We now study a few theoretical properties of policy poisoning on TCE.
All proofs are in the appendix.
First of all, the attack is always feasible.

\begin{restatable}{proposition}{feasibility}
\label{feasibility}
The attack problem~\eqref{attack:formulation}-\eqref{attack:c3} is always feasible for any target policy $\pi^\dagger$.
\end{restatable}
Proposition \ref{feasibility} states that for any target policy $\pi^\dagger$, there exists a perturbation on the rewards that teaches the learner that policy. Therefore, the attacker changing $r$'s but not $s,a,s'$ is already sufficient for policy poisoning.

We next bound the attack cost.
Let the MDP estimated on the clean data be $\hat M^0=(\S,\A,\hat P, \hat R^0, \gamma)$.
Let $Q^0$ be the $Q$ function that satisfies the Bellman optimality equation on $\hat M^0$. Define $\Delta(\epsilon)=\max_{s\in\S}[\max_{a\neq \pi^\dagger(s)} Q^0(s,a)-Q^0(s,\pi^\dagger(s))+\epsilon]_+$, where $[]_+$ takes the maximum over 0. 
Intuitively, $\Delta(\epsilon)$ measures how suboptimal the target policy $\pi^\dagger$ is compared to the clean optimal policy $\pi^0$ learned on $\hat M^0$, up to a margin parameter $\epsilon$.
\begin{restatable}{theorem}{TCEboundcost}
\label{TCE:bound-cost}
Assume $\alpha\ge 1$ in~\eqref{attack:objective}. Let $\bv r^*$, $\hat R^*$ and $Q^*$ be an optimal solution to~\eqref{attack:formulation}-\eqref{attack:c3}, then
\begin{equation}
\frac{1}{2}(1-\gamma)\Delta(\epsilon)\left(\min_{s,a} {|T_{s,a}|}\right)^\frac{1}{\alpha}\le \|\bv r^*-\bv r^0\|_\alpha\le \frac{1}{2}(1+\gamma)\Delta(\epsilon) T^{\frac{1}{\alpha}}.
\end{equation}
\end{restatable}

\begin{corollary}
If $\alpha=1$, then the optimal attack cost is $O(\Delta(\epsilon) T)$. If $\alpha=2$, then the optimal attack cost is $O(\Delta(\epsilon) \sqrt{T})$. If $\alpha=\infty$, then the optimal attack cost is $O(\Delta(\epsilon))$.
\end{corollary}
Note that both the upper and lower bounds on the attack cost are linear with respect to $\Delta(\epsilon)$, which can be estimated directly from the clean training set $D^0$. This allows the attacker to easily estimate its attack cost before actually solving the attack problem.

\subsection{Poisoning a Linear Quadratic Regulator (LQR) Victim}\label{sec:poison_LQR}
As the second example, we study an LQR victim that performs system identification from a batch training set~\cite{dean2017sample}.
Let the linear dynamical system be
\begin{equation}
s_{t+1} =As_t+Ba_t+w_t, \forall t\ge 0\label{eq:linear_dynamics},
\end{equation}
where $A\in \R^{n\times n}, B\in\R^{n\times m}$, $s_t\in\R^n$ is the state, $a_t\in\R^m$ is the control signal, and $w_t\sim\mathcal{N}(\bv 0,\sigma^2 I)$ is a Gaussian noise.
When the agent takes action $a$ at state $s$, it suffers a quadratic loss of the general form 
\begin{equation}\label{LQR:loss}
L(s,a)=\frac{1}{2}s^\top Qs+q^\top s+a^\top Ra+c
\end{equation}
for some $Q\succeq0$, $R\succ0$, $q\in \R^n$ and $c\in \R$. 
Here we have redefined the symbols $Q$ and $R$ in order to conform with the notation convention in LQR:
now we use $Q$ for the quadratic loss matrix associated with state, not the action-value function; we use $R$ for the quadratic loss matrix associated with action, not the reward function.
The previous reward function $R(s,a)$ in general MDP (section~\ref{prelim}) is now equivalent to the negative loss $-L(s,a)$.
This form of loss captures various LQR control problems.
Note that the above linear dynamical system can be viewed as an MDP with transition kernel $P(s^\prime\mid s,a)=\mathcal{N}(As+Ba,\sigma^2I)$ and reward function $-L(s,a)$.
The environment is thus characterized by matrices $A$, $B$ (for transition kernel) and $Q$, $R$, $q$, $c$ (for reward function), which are all unknown to the learner.

We assume the clean training data $D^0=(s_t,a_t,r_t^0, s_{t+1})_{0:T-1}$ was generated by running the linear system for multiple episodes following some random policy~\cite{dean2017sample}.
Let the poisoned data be $D=(s_t,a_t,r_t, s_{t+1})_{0:T-1}$.
Instantiating model estimation~\eqref{MLE},~\eqref{MMSE}, the learner performs system identification on the poisoned data:
\begin{equation}\label{LQR:estimate-transition}
(\hat{A}, \hat{B}) \in \argmin _{(A, B)}  \frac{1}{2}\sum_{t=0}^{T-1}\left\|A s_{t}+B a_{t}-s_{t+1} \right\|_{2}^{2}
\end{equation}
\begin{equation}\label{LQR:estimate-reward}
(\hat{Q}, \hat{R},\hat q,\hat c)= \argmin _{(Q\succeq 0, R\succeq \epsilon I,q,c)}  \frac{1}{2}\sum_{t=0}^{T-1}\left\|\frac{1}{2}s_t^\top Qs_t+q^\top s_t+a_t^\top Ra_t+c+r_t\right\|_{2}^{2}.
\end{equation}
Note that in~\eqref{LQR:estimate-reward}, the learner uses a stronger constraint $R\succeq \epsilon I$ than the original constraint $R\succ0$, which guarantees that the minimizer can be achieved. The conditions to further guarantee~\eqref{LQR:estimate-reward} having a unique solution depend on the property of certain matrices formed by the clean training set $D^0$, which we defer to appendix~\ref{LQR:condition_uniqueness}.

The learner then computes the optimal control policy with respect to $\hat A$, $\hat B$, $\hat Q$, $\hat R$, $\hat q$ and $\hat c$.
We assume the learner solves a discounted version of LQR control
\begin{eqnarray}\label{LQR}
  \max_{\pi:\S\mapsto \A}  &&-\mathbb{E}\left[\sum_{\tau=0}^\infty\gamma^\tau (\frac{1}{2}s_\tau^\top \hat Qs_\tau+\hat q^\top s_\tau+\pi(s_\tau)^\top \hat R\pi(s_\tau)+\hat c)\right]\label{objective-LQR}\\ 
  \text { s.t. } &&  s_{\tau+1} =\hat As_\tau+\hat B\pi(s_\tau)+w_\tau, \forall \tau\ge0.
\end{eqnarray}
where the expectation is over $w_\tau$.
It is known that the control problem has a closed-form solution $\hat a_\tau=\hat\pi(s_\tau)=Ks_\tau+k$, where
\begin{equation}\label{LQR:K}
K=-\gamma\left(\hat R+\gamma \hat B^{\top} X \hat B\right)^{-1} \hat B^{\top} X \hat A,\quad k = - \gamma(\hat R+\gamma \hat B^{\top}X\hat B)^{-1} \hat B^{\top} x.
\end{equation}
Here $X\succeq 0$ is the unique solution of the Algebraic Riccati Equation,
\begin{equation}\label{LQR:Riccatti}
X=\gamma \hat A^{\top} X \hat A-\gamma^{2} \hat A^{\top} X \hat B\left(\hat R+\gamma \hat B^{\top} X \hat B\right)^{-1} \hat B^{\top} X \hat A+\hat Q,
\end{equation}
and $x$ is a vector that satisfies
\begin{equation}\label{LQR:Riccatti-2}
x=\hat q+\gamma (\hat A+\hat BK)^\top x.
\end{equation}

The attacker aims to force the victim into taking target action $\pi^\dagger(s), \forall s\in\R^n$. 
Note that in LQR, the attacker cannot arbitrarily choose $\pi^\dagger$, as the optimal control policy $K$ and $k$ enforce a linear structural constraint between $\pi^\dagger(s)$ and $s$. One can easily see that the target action must obey $\pi^\dagger(s)=K^\dagger s+k^\dagger$ for some $(K^\dagger,k^\dagger)$ in order to achieve successful attack. 
Therefore we must  assume instead that the attacker has a target policy specified by a pair $(K^\dagger, k^\dagger)$. 
However, an arbitrarily linear policy may still not be feasible. 
A target policy $(K^\dagger, k^\dagger)$ is feasible if and only if it is produced by solving some Riccati equation, namely, it must lie in the following set:
\begin{equation}
\{(K,k): \exists Q\succeq 0,R\succeq \epsilon I, q\in\R^n, c\in\R, \text{such that}~\eqref{LQR:K},~\eqref{LQR:Riccatti},\text{ and}~\eqref{LQR:Riccatti-2} \text{ are satisfied}\}.
\end{equation}
Therefore, to guarantee feasibility, we assume the attacker always picks the target policy $(K^\dagger, k^\dagger)$ by solving an LQR problem with some attacker-defined loss function.
We can now pose the policy poisoning attack problem: 
\begin{eqnarray}\label{attack:LQR}
  \min_{\bv r,  \hat Q, \hat R, \hat q, \hat c, X, x}  &&\|\bv r-\bv r^0\|_\alpha\label{attack:objective-LQR}\\ 
  \text { s.t. }
  &&-\gamma\left(\hat R+\gamma \hat B^{\top} X \hat B\right)^{-1} \hat B^{\top} X \hat A=K^\dagger\label{attack-LQR-c1}\\
  && -\gamma\left(\hat R+\gamma \hat B^{\top} X \hat B\right)^{-1} \hat B^{\top} x = k^\dagger\label{attack-LQR-c2}\\
  && X=\gamma \hat A^{\top} X \hat A-\gamma^{2} \hat A^{\top} X \hat B\left(\hat R+\gamma \hat B^{\top} X \hat B\right)^{-1} \hat B^{\top} X \hat A+\hat Q\label{attack-LQR-c3}\\
   && x =   \hat q +\gamma(\hat A+\hat BK^\dagger)^\top x\label{attack-LQR-c4}\\
  &&(\hat{Q}, \hat{R},\hat q,\hat c) =\argmin _{(Q\succeq 0, R\succeq \epsilon I,q,c)}  \sum_{t=0}^{T-1}\left\|\frac{1}{2}s_t^\top Qs_t+q^\top s_t+a_t^\top Ra_t+c+r_t\right\|_{2}^{2}
\label{attack-LQR-c5}\\
  &&X\succeq 0\label{attack-LQR-c6}.
\end{eqnarray}
Note that the estimated transition matrices $\hat A$, $\hat B$ are not optimization variables because the attacker can only modify the rewards, which will not change the learner's estimate on $\hat A$ and $\hat B$.
The attack optimization~\eqref{attack:objective-LQR}-\eqref{attack-LQR-c6} is hard to solve due to the constraint~\eqref{attack-LQR-c5} itself being a semi-definite program (SDP). 
To overcome the difficulty, we pull all the positive semi-definite constraints out of the lower-level optimization. This leads to a more stringent surrogate attack optimization (see appendix~\ref{convex:surrogate}). 
Solving the surrogate attack problem, whose feasible region is a subset of the original problem, in general gives a suboptimal solution to~\eqref{attack:objective-LQR}-\eqref{attack-LQR-c6}. But it comes with one advantage: convexity.

\section{Experiments}
Throughout the experiments, we use CVXPY~\cite{cvxpy} to implement the optimization. 
All code can be found in \href{https://github.com/myzwisc/PPRL\_NeurIPS19}{https://github.com/myzwisc/PPRL\_NeurIPS19}.

\subsection{Policy Poisoning Attack on TCE Victim}

\textbf{Experiment 1.}
We consider a simple MDP with two states $A,B$ and two actions: \textit{stay} in the same state or \textit{move} to the other state, shown in figure \ref{fig:toy-MDP}. 
The discounting factor is $\gamma=0.9$. 
The MDP's $Q$ values are shown in table~\ref{tab:toy-MDP-original-Q}. Note that the optimal policy will always pick action \textit{stay}.
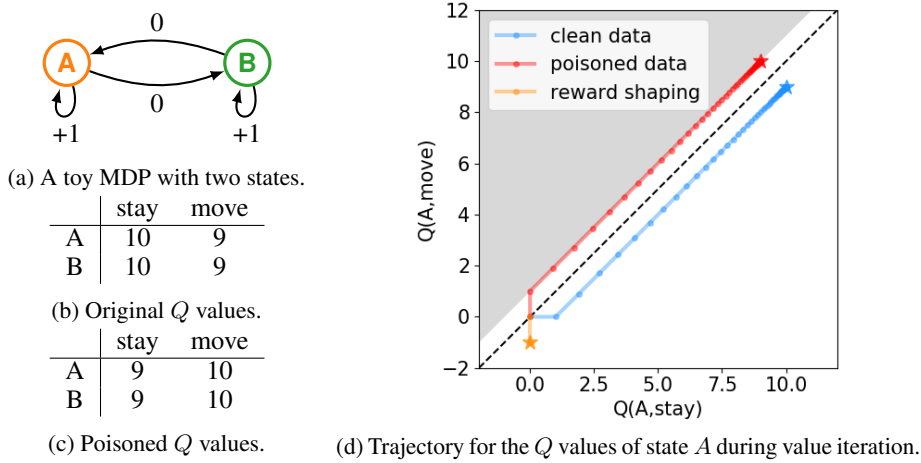
\begin{figure}[]
\begin{subfigure}[c]{0.33\textwidth}
  \begin{subfigure}{1\textwidth}
  \centering
  \begin{tikzpicture}
        	\tikzstyle{n} = [very thick,circle,inner sep=0mm,minimum width=6mm]
        	\tikzstyle{a} = [thick,>=latex,->]
        	\def\dx{1.2}
        	\def\dy{-1.2}
        	\node[n,C1,draw=C1] (2) at (\dy,0) {\textbf{\textsf{A}}};
        	\node[n,C2,draw=C2] (1) at (\dx,0) {\textbf{\textsf{B}}};
        	\path[a]
        	(2) edge [loop below] node {+1}(2) 
        	(1) edge [loop below] node {+1}(1) 
        	(2) edge [bend right=20] node[below] {0}(1)
        	(1) edge [bend right=20] node[above] {0}(2);
        	\end{tikzpicture}
  \caption{A toy MDP with two states.}%
  \label{fig:toy-MDP}
\end{subfigure}
\begin{subfigure}{1\textwidth}
	\centering
	\begin{tabular}{ c | c  c  c}
		& stay & move\\ \hline
		A & 10 & 9 \\
		B & 10 & 9
	\end{tabular}
	\caption{Original $Q$ values.}%
	\label{tab:toy-MDP-original-Q}
\end{subfigure}
\begin{subfigure}{1\textwidth}
	\centering
	\begin{tabular}{ c | c  c  c}
		& stay & move\\ \hline
		A & 9 & 10 \\
		B & 9 & 10
	\end{tabular}
	\caption{Poisoned $Q$ values.}%
	\label{tab:toy-MDP-poisoned-Q}
\end{subfigure}
\end{subfigure}
\begin{subfigure}[c]{.555\textwidth}
	\centering
	\includegraphics[width=0.73\textwidth]{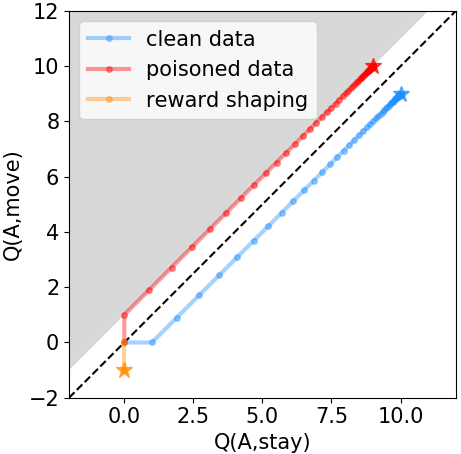}
	\caption{Trajectory for the $Q$ values of state $A$ during value iteration.}
	\label{fig:traj-toy-MDP}
\end{subfigure}%
\caption{Poisoning TCE in a two-state MDP.}
\end{figure}
The clean training data $D^0$ reflects this underlying MDP, and consists of 4 tuples:
\begin{eqnarray}
(A,\textit{stay},1,A)\quad(A,\textit{move},0,B)\quad (B,\textit{stay},1,B)\quad(B,\textit{move},0,A)\nonumber
\end{eqnarray}
Let the attacker's target policy be $\pi^\dagger(s)=$\textit{move}, for any state $s$. 
The attacker sets $\epsilon=1$ and uses $\alpha=2$, i.e. $\|\bv r - \bv r^0\|_2$ as the attack cost. 
Solving the policy poisoning attack optimization problem~\eqref{attack:formulation}-\eqref{attack:c3} produces the poisoned data:
\begin{eqnarray}
(A,\textit{stay},0,A)
\quad(A,\textit{move},1,B) \quad(B,\textit{stay},0,B)
\quad(B,\textit{move},1,A)\nonumber
\end{eqnarray}
with attack cost $\|\bv r - \bv r^0\|_2 = 2$. The resulting poisoned $Q$ values are shown in table~\ref{tab:toy-MDP-poisoned-Q}.
To verify this attack, we run TCE learner on both clean data and poisoned data.
Specifically, we estimate the transition kernel and the reward function as in~\eqref{eq:tabular_transition} and~\eqref{eq:tabular_reward} on each data set, and then run value iteration until the $Q$ values converge. 
In Figure~\ref{fig:traj-toy-MDP}, we show the trajectory of $Q$ values for state $A$, where the $x$ and $y$ axes denote $Q(A,stay)$ and $Q(A, move)$ respectively.
All trajectories start at $(0,0)$.
The dots on the trajectory correspond to each step of value iteration, while the star denotes the converged $Q$ values.
The diagonal dashed line is the (zero margin) policy boundary, while the gray region is the
$\epsilon$-robust target $Q$ polytope with an offset $\epsilon=1$ to the policy boundary. 
The trajectory of clean data converges to a point below the policy boundary, where the action $stay$ is optimal.
With the poisoned data, the trajectory of $Q$ values converge to a point exactly on the boundary of the $\epsilon$-robust target $Q$ polytope, where the action $move$ becomes optimal. 
This validates our attack. 

We also compare our attack with reward shaping~\cite{ng1999policy}.
We let the potential function $\phi(s)$ be the optimal value function $V(s)$ for all $s$ to shape the clean dataset. The dataset after shaping is
\begin{eqnarray}
(A,\textit{stay},0,A)\quad(A,\textit{move},-1,B)\quad (B,\textit{stay},0,B)\quad(B,\textit{move},-1,A)\nonumber
\end{eqnarray}
In Figure~\ref{fig:traj-toy-MDP}, we show the trajectory of $Q$ values after reward shaping. Note that same as on clean dataset, the trajectory after shaping converges to a point also below the policy boundary. This means reward shaping can not make the learner learn a different policy from the original optimal policy. Also note that after reward shaping, value iteration converges much faster (in only one iteration), which matches the benefits of reward shaping shown in \cite{ng1999policy}. More importantly, this illustrates the difference between our attack and reward shaping.

\begin{figure}[h]
	\begin{subfigure}[t]{.48\textwidth}
		\centering
		\resizebox{1\textwidth}{!}{%
		\begin{tikzpicture}
		\draw[step=2cm,black,thin,opacity=0.5] (0,0) grid (14,12);
		\fill[black] (0,2) rectangle (2,12);
		\fill[black] (4,0) rectangle (14,2);
		\fill[black] (4,4) rectangle (10,8);
		\fill[black] (4,10) rectangle (8,12);
		\fill[black] (12,0) rectangle (14,12);
		\fill[black] (10,10) rectangle (12,12);
		\fill[gray,opacity=0.8] (2,4) rectangle (4,8);
		\fill[blue,opacity=0.5] (2,10) rectangle (4,12);
		\node at (1,1) {\LARGE S};
		\node at (3,11) {\LARGE G};
		
		\draw (1,-0.5) -- (1,-2.5) -- (3,-2.5) -- (3,-0.5) -- (1,-0.5);
		\fill[blue,opacity=0.5] (1,-2.5) rectangle (3,-0.5);
		\node at (2,-1.5) {\LARGE G};
		\node[text width=1cm] at (3.6,-1.5) {\LARGE:2};
		
		\draw (6,-0.5) -- (6,-2.5) -- (8,-2.5) -- (8,-0.5) -- (8,-0.5);
		\fill[gray,opacity=0.8] (6,-2.5) rectangle (8,-0.5);
		\node[text width=1cm] at (8.8,-1.5) {\LARGE:$-10$};
		
		\draw (10.6,-0.5) -- (10.6,-2.5) -- (12.6,-2.5) -- (12.6,-0.5) -- (10.6,-0.5);
		\fill[white] (10.6,-2.5) rectangle (12.6,-0.5);
		\node[text width=1cm] at (13.2,-1.5) {\LARGE:$-1$};
		
		\draw [->=stealth,line width=0.5mm, opacity=0.5, blue]
		(1.25 , 0.75) --(3, 0.75)
		(3,0.75) -- (3,3)
		(3,3) -- (11,3)
		(11,3) -- (11,9)
		(11,9) -- (3,9)
		(3,9) -- (3,10);
		\draw [->=stealth,line width=0.5mm,opacity=0.5, red]
		(1.25 , 1.25) --(2.75, 1.25)
		(2.75,1.25) -- (2.75,10);

		\draw [orange,->,>=stealth] (0.5,1) -- (-0.5,1);
		\draw [orange,->,>=stealth] (1,0.5) -- (1,-0.5);
		\draw [orange,->,>=stealth] (1.5,1.2) -- (2.5,1.2);
		\draw [orange,->,>=stealth] (2.5,0.8) -- (1.5,0.8);
		\draw [orange,->,>=stealth] (1,1.5) -- (1,2.5);
		
		\draw [orange,->,>=stealth] (3,0.5) -- (3,-0.5);
		\draw [orange,->,>=stealth] (3.2,2.5) -- (3.2,1.5);
		\draw [orange,->,>=stealth] (3.5,1) -- (4.5,1);
		\draw [orange,->,>=stealth] (2.8,1.5) -- (2.8,2.5);
		
		\draw [orange,->,>=stealth] (3.2,4.5) -- (3.2,3.5);
		\draw [orange,->,>=stealth] (3.5,3.2) -- (4.5,3.2);
		\draw [orange,->,>=stealth] (4.5,2.8) -- (3.5,2.8);
		\draw [orange,->,>=stealth] (2.8,3.5) -- (2.8,4.5);
		\draw [orange,->,>=stealth] (2.5,3) -- (1.5,3);
		
		\draw [orange,->,>=stealth] (5.5,3.2) -- (6.5,3.2);
		\draw [orange,->,>=stealth] (6.5,2.8) -- (5.5,2.8);
		\draw [orange,->,>=stealth] (5,3.5) -- (5,4.5);
		\draw [orange,->,>=stealth] (5,2.5) -- (5,1.5);
		
		\draw [orange,->,>=stealth] (7.5,3.2) -- (8.5,3.2);
		\draw [orange,->,>=stealth] (8.5,2.8) -- (7.5,2.8);
		\draw [orange,->,>=stealth] (7,3.5) -- (7,4.5);
		\draw [orange,->,>=stealth] (7,2.5) -- (7,1.5);
		
		\draw [orange,->,>=stealth] (9.5,3.2) -- (10.5,3.2);
		\draw [orange,->,>=stealth] (10.5,2.8) -- (9.5,2.8);
		\draw [orange,->,>=stealth] (9,3.5) -- (9,4.5);
		\draw [orange,->,>=stealth] (9,2.5) -- (9,1.5);
		
		\draw [orange,->,>=stealth] (11.2,4.5) -- (11.2,3.5);
		\draw [orange,->,>=stealth] (11.5,3) -- (12.5,3);
		\draw [orange,->,>=stealth] (10.8,3.5) -- (10.8,4.5);
		\draw [orange,->,>=stealth] (11,2.5) -- (11,1.5);
		
		\draw [orange,->,>=stealth] (11.2,6.5) -- (11.2,5.5);
		\draw [orange,->,>=stealth] (11.5,5) -- (12.5,5);
		\draw [orange,->,>=stealth] (10.8,5.5) -- (10.8,6.5);
		\draw [orange,->,>=stealth] (10.5,5) -- (9.5,5);
		
		\draw [orange,->,>=stealth] (11.2,8.5) -- (11.2,7.5);
		\draw [orange,->,>=stealth] (11.5,7) -- (12.5,7);
		\draw [orange,->,>=stealth] (10.8,7.5) -- (10.8,8.5);
		\draw [orange,->,>=stealth] (10.5,7) -- (9.5,7);
		
		\draw [orange,->,>=stealth] (11.5,9) -- (12.5,9);
		\draw [orange,->,>=stealth] (11,9.5) -- (11,10.5);
		\draw [orange,->,>=stealth] (9.5,9.2) -- (10.5,9.2);
		\draw [orange,->,>=stealth] (10.5,8.8) -- (9.5,8.8);
		
		\draw [orange,->,>=stealth] (8.8,9.5) -- (8.8,10.5);
		\draw [orange,->,>=stealth] (9.2,10.5) -- (9.2,9.5);
		\draw [orange,->,>=stealth] (7.5,9.2) -- (8.5,9.2);
		\draw [orange,->,>=stealth] (8.5,8.8) -- (7.5,8.8);
		\draw [orange,->,>=stealth] (9,8.5) -- (9,7.5);
		
		\draw [orange,->,>=stealth] (7,9.5) -- (7,10.5);
		\draw [orange,->,>=stealth] (5.5,9.2) -- (6.5,9.2);
		\draw [orange,->,>=stealth] (6.5,8.8) -- (5.5,8.8);
		\draw [orange,->,>=stealth] (7,8.5) -- (7,7.5);
		
		\draw [orange,->,>=stealth] (5,9.5) -- (5,10.5);
		\draw [orange,->,>=stealth] (3.5,9.2) -- (4.5,9.2);
		\draw [orange,->,>=stealth] (4.5,8.8) -- (3.5,8.8);
		\draw [orange,->,>=stealth] (5,8.5) -- (5,7.5);
		
		\draw [orange,->,>=stealth] (2.8,9.5) -- (2.8,10.5);
		\draw [orange,->,>=stealth] (3.2,10.5) -- (3.2,9.5);
		\draw [orange,->,>=stealth] (2.5,9) -- (1.5,9);
		\draw [orange,->,>=stealth] (2.8,7.5) -- (2.8,8.5);
		\draw [orange,->,>=stealth] (3.2,8.5) -- (3.2,7.5);
		
		\draw [orange,->,>=stealth] (2.8,5.5) -- (2.8,6.5);
		\draw [orange,->,>=stealth] (3.2,6.5) -- (3.2,5.5);
		\draw [orange,->,>=stealth] (2.5,7) -- (1.5,7);
		\draw [orange,->,>=stealth] (3.5,7) -- (4.5,7);
		
		\draw [orange,->,>=stealth] (2.5,5) -- (1.5,5);
		\draw [orange,->,>=stealth] (3.5,5) -- (4.5,5);
		
		\draw [orange,->,>=stealth] (2.5,11) -- (1.5,11);
		\draw [orange,->,>=stealth] (3.5,11) -- (4.5,11);
		\draw [orange,->,>=stealth] (3,11.5) -- (3,12.5);
		
		\draw [orange,->,>=stealth] (8.5,11) -- (7.5,11);
		\draw [orange,->,>=stealth] (9.5,11) -- (10.5,11);
		\draw [orange,->,>=stealth] (9,11.5) -- (9,12.5);
		\node at (4,3.2) {\textcolor{blue}{{\large-0.572}}}; 
		\node at (2.5,4) {\textcolor{red}{{\large+0.572}}}; 
		\node at (6,3.2) {\textcolor{blue}{{\large-0.515}}}; 
		\node at (8,3.2) {\textcolor{blue}{{\large-0.464}}}; 
		\node at (10,3.2) {\textcolor{blue}{{\large-0.417}}}; 
		\node at (10.6,4) {\textcolor{blue}{{\large-0.376}}}; 
		\node at (10.6,6) {\textcolor{blue}{{\large-0.338}}}; 
		\node at (10.6,8) {\textcolor{blue}{{\large-0.304}}}; 
		\node at (10,8.8) {\textcolor{blue}{{\large-0.274}}}; 
		\node at (8,8.8) {\textcolor{blue}{{\large-0.246}}}; 
		\node at (6,8.8) {\textcolor{blue}{{\large-0.221}}}; 
		\node at (4,8.8) {\textcolor{blue}{{\large-0.200}}}; 
		\node at (2.5,10) {\textcolor{red}{{\large+0.238}}}; 
		\node at (3,12) {\textcolor{red}{{\large+2.139}}}; 
		\node at (2.5,8) {\textcolor{red}{{\large+0.464}}}; 
		\node at (2.5,6) {\textcolor{red}{{\large+0.515}}}; 
		\end{tikzpicture}
		}
		\caption{Grid world with a single terminal state $G$.}
		\label{fig:grid-world-1}
	\end{subfigure}
	\begin{subfigure}[t]{.5\textwidth}
		\centering	
		\resizebox{1\textwidth}{!}{%
		\begin{tikzpicture}
		\draw (1,-0.5) -- (1,-2) -- (2.5,-2) -- (2.5,-0.5) -- (1,-0.5);
		\fill[blue,opacity=0.5] (1,-2) rectangle (2.5,-0.5);
		\node at (1.75,-1.25) {\Large G1};
		\node[text width=0.5cm] at (2.9,-1.25) {\Large :1};
				
		\draw (5.2,-0.5) -- (5.2,-2) -- (6.7,-2) -- (6.7,-0.5) -- (5.2,-0.5);
		\fill[blue,opacity=0.5] (5.2,-2) rectangle (6.7,-0.5);
		\node at (5.95,-1.25) {\Large G2};
		\node[text width=0.7cm] at (7.15,-1.25) {\Large:2};
		
		\draw (9,-0.5) -- (9,-2) -- (10.5,-2) -- (10.5,-0.5) -- (9,-0.5);
		\fill[white] (9,-2) rectangle (10.5,-0.5);
		\node[text width=0.7cm] at (11,-1.25) {\Large:$-1$};
		
		\draw[step=2cm,black,thin,opacity=0.5] (0,0) grid (12,12);
		\fill[black] (2,8) rectangle (4,12);
		\fill[black] (8,8) rectangle (10,12);
		\node at (1,11) {\Large G1};
		\node at (11,11) {\Large G2};
		\fill[blue,opacity=0.5] (0,10) rectangle (2,12);
		\fill[blue,opacity=0.5] (10,10) rectangle (12,12);
		\node at (3,7) {\Large S};
		
		\tikzset{>=latex}
		\draw [->=stealth,line width=0.5mm, opacity=0.5,blue]
		(2.75 , 7) --(1, 7)
		(1 , 7) --(1, 10);
		\draw [->=stealth,line width=0.5mm, opacity=0.5,red]
		(3.25, 7) --(11,7)
		(11, 7) -- (11, 10);

		\draw [orange,->,>=stealth] (0.5,1) -- (-0.5,1);
		\draw [orange,->,>=stealth] (1,0.5) -- (1,-0.5);
		\draw [orange,->,>=stealth] (1.5,1.2) -- (2.5,1.2);
		\draw [orange,->,>=stealth] (2.5,0.8) -- (1.5,0.8);
		\draw [orange,->,>=stealth] (0.8,1.5) -- (0.8,2.5);
		\draw [orange,->,>=stealth] (1.2,2.5) -- (1.2,1.5);
		
		\draw [orange,->,>=stealth] (0.5,3) -- (-0.5,3);
		\draw [orange,->,>=stealth] (1.5,3.2) -- (2.5,3.2);
		\draw [orange,->,>=stealth] (2.5,2.8) -- (1.5,2.8);
		\draw [orange,->,>=stealth] (0.8,3.5) -- (0.8,4.5);
		\draw [orange,->,>=stealth] (1.2,4.5) -- (1.2,3.5);
		
		\draw [orange,->,>=stealth] (0.5,5) -- (-0.5,5);
		\draw [orange,->,>=stealth] (1.5,5.2) -- (2.5,5.2);
		\draw [orange,->,>=stealth] (2.5,4.8) -- (1.5,4.8);
		\draw [orange,->,>=stealth] (0.8,5.5) -- (0.8,6.5);
		\draw [orange,->,>=stealth] (1.2,6.5) -- (1.2,5.5);
		
		\draw [orange,->,>=stealth] (0.5,7) -- (-0.5,7);
		\draw [orange,->,>=stealth] (1.5,7.2) -- (2.5,7.2);
		\draw [orange,->,>=stealth] (2.5,6.8) -- (1.5,6.8);
		\draw [orange,->,>=stealth] (0.8,7.5) -- (0.8,8.5);
		\draw [orange,->,>=stealth] (1.2,8.5) -- (1.2,7.5);
		
		\draw [orange,->,>=stealth] (0.5,9) -- (-0.5,9);
		\draw [orange,->,>=stealth] (1.5,9) -- (2.5,9);
		\draw [orange,->,>=stealth] (0.8,9.5) -- (0.8,10.5);
		\draw [orange,->,>=stealth] (1.2,10.5) -- (1.2,9.5);
		
		\draw [orange,->,>=stealth] (0.5,11) -- (-0.5,11);
		\draw [orange,->,>=stealth] (1.5,11) -- (2.5,11);
		\draw [orange,->,>=stealth] (1,11.5) -- (1,12.5);
		
		\draw [orange,->,>=stealth] (3,0.5) -- (3,-0.5);
		\draw [orange,->,>=stealth] (3.5,1.2) -- (4.5,1.2);
		\draw [orange,->,>=stealth] (4.5,0.8) -- (3.5,0.8);
		\draw [orange,->,>=stealth] (2.8,1.5) -- (2.8,2.5);
		\draw [orange,->,>=stealth] (3.2,2.5) -- (3.2,1.5);
		
		\draw [orange,->,>=stealth] (5,0.5) -- (5,-0.5);
		\draw [orange,->,>=stealth] (5.5,1.2) -- (6.5,1.2);
		\draw [orange,->,>=stealth] (6.5,0.8) -- (5.5,0.8);
		\draw [orange,->,>=stealth] (4.8,1.5) -- (4.8,2.5);
		\draw [orange,->,>=stealth] (5.2,2.5) -- (5.2,1.5);
		
		\draw [orange,->,>=stealth] (7,0.5) -- (7,-0.5);
		\draw [orange,->,>=stealth] (7.5,1.2) -- (8.5,1.2);
		\draw [orange,->,>=stealth] (8.5,0.8) -- (7.5,0.8);
		\draw [orange,->,>=stealth] (6.8,1.5) -- (6.8,2.5);
		\draw [orange,->,>=stealth] (7.2,2.5) -- (7.2,1.5);
		
		\draw [orange,->,>=stealth] (9,0.5) -- (9,-0.5);
		\draw [orange,->,>=stealth] (9.5,1.2) -- (10.5,1.2);
		\draw [orange,->,>=stealth] (10.5,0.8) -- (9.5,0.8);
		\draw [orange,->,>=stealth] (8.8,1.5) -- (8.8,2.5);
		\draw [orange,->,>=stealth] (9.2,2.5) -- (9.2,1.5);
		
		\draw [orange,->,>=stealth] (11,0.5) -- (11,-0.5);
		\draw [orange,->,>=stealth] (11.5,1) -- (12.5,1);
		\draw [orange,->,>=stealth] (10.8,1.5) -- (10.8,2.5);
		\draw [orange,->,>=stealth] (11.2,2.5) -- (11.2,1.5);
		
		\draw [orange,->,>=stealth] (3.5,3.2) -- (4.5,3.2);
		\draw [orange,->,>=stealth] (4.5,2.8) -- (3.5,2.8);
		\draw [orange,->,>=stealth] (2.8,3.5) -- (2.8,4.5);
		\draw [orange,->,>=stealth] (3.2,4.5) -- (3.2,3.5);
		
		\draw [orange,->,>=stealth] (3.5,5.2) -- (4.5,5.2);
		\draw [orange,->,>=stealth] (4.5,4.8) -- (3.5,4.8);
		\draw [orange,->,>=stealth] (2.8,5.5) -- (2.8,6.5);
		\draw [orange,->,>=stealth] (3.2,6.5) -- (3.2,5.5);
		
		\draw [orange,->,>=stealth] (3.5,7.2) -- (4.5,7.2);
		\draw [orange,->,>=stealth] (4.5,6.8) -- (3.5,6.8);
		\draw [orange,->,>=stealth] (3,7.5) -- (3,8.5);
		
		\draw [orange,->,>=stealth] (5.5,3.2) -- (6.5,3.2);
		\draw [orange,->,>=stealth] (6.5,2.8) -- (5.5,2.8);
		\draw [orange,->,>=stealth] (4.8,3.5) -- (4.8,4.5);
		\draw [orange,->,>=stealth] (5.2,4.5) -- (5.2,3.5);
		
		\draw [orange,->,>=stealth] (5.5,5.2) -- (6.5,5.2);
		\draw [orange,->,>=stealth] (6.5,4.8) -- (5.5,4.8);
		\draw [orange,->,>=stealth] (4.8,5.5) -- (4.8,6.5);
		\draw [orange,->,>=stealth] (5.2,6.5) -- (5.2,5.5);
		
		\draw [orange,->,>=stealth] (5.5,7.2) -- (6.5,7.2);
		\draw [orange,->,>=stealth] (6.5,6.8) -- (5.5,6.8);
		\draw [orange,->,>=stealth] (4.8,7.5) -- (4.8,8.5);
		\draw [orange,->,>=stealth] (5.2,8.5) -- (5.2,7.5);
		
		\draw [orange,->,>=stealth] (5.5,9.2) -- (6.5,9.2);
		\draw [orange,->,>=stealth] (6.5,8.8) -- (5.5,8.8);
		\draw [orange,->,>=stealth] (4.8,9.5) -- (4.8,10.5);
		\draw [orange,->,>=stealth] (5.2,10.5) -- (5.2,9.5);
		\draw [orange,->,>=stealth] (4.5,9) -- (3.5,9);
		
		\draw [orange,->,>=stealth] (5.5,11.2) -- (6.5,11.2);
		\draw [orange,->,>=stealth] (6.5,10.8) -- (5.5,10.8);
		\draw [orange,->,>=stealth] (5,11.5) -- (5,12.5);
		\draw [orange,->,>=stealth] (4.5,11) -- (3.5,11);
		
		\draw [orange,->,>=stealth] (7.5,3.2) -- (8.5,3.2);
		\draw [orange,->,>=stealth] (8.5,2.8) -- (7.5,2.8);
		\draw [orange,->,>=stealth] (6.8,3.5) -- (6.8,4.5);
		\draw [orange,->,>=stealth] (7.2,4.5) -- (7.2,3.5);
		
		\draw [orange,->,>=stealth] (7.5,5.2) -- (8.5,5.2);
		\draw [orange,->,>=stealth] (8.5,4.8) -- (7.5,4.8);
		\draw [orange,->,>=stealth] (6.8,5.5) -- (6.8,6.5);
		\draw [orange,->,>=stealth] (7.2,6.5) -- (7.2,5.5);
		
		\draw [orange,->,>=stealth] (7.5,7.2) -- (8.5,7.2);
		\draw [orange,->,>=stealth] (8.5,6.8) -- (7.5,6.8);
		\draw [orange,->,>=stealth] (6.8,7.5) -- (6.8,8.5);
		\draw [orange,->,>=stealth] (7.2,8.5) -- (7.2,7.5);
		
		\draw [orange,->,>=stealth] (7.5,9) -- (8.5,9);
		\draw [orange,->,>=stealth] (6.8,9.5) -- (6.8,10.5);
		\draw [orange,->,>=stealth] (7.2,10.5) -- (7.2,9.5);
		
		\draw [orange,->,>=stealth] (7.5,11) -- (8.5,11);
		\draw [orange,->,>=stealth] (7,11.5) -- (7,12.5);
		
		\draw [orange,->,>=stealth] (9.5,3.2) -- (10.5,3.2);
		\draw [orange,->,>=stealth] (10.5,2.8) -- (9.5,2.8);
		\draw [orange,->,>=stealth] (8.8,3.5) -- (8.8,4.5);
		\draw [orange,->,>=stealth] (9.2,4.5) -- (9.2,3.5);
		
		\draw [orange,->,>=stealth] (9.5,5.2) -- (10.5,5.2);
		\draw [orange,->,>=stealth] (10.5,4.8) -- (9.5,4.8);
		\draw [orange,->,>=stealth] (8.8,5.5) -- (8.8,6.5);
		\draw [orange,->,>=stealth] (9.2,6.5) -- (9.2,5.5);
		
		\draw [orange,->,>=stealth] (9.5,7.2) -- (10.5,7.2);
		\draw [orange,->,>=stealth] (10.5,6.8) -- (9.5,6.8);
		\draw [orange,->,>=stealth] (9,7.5) -- (9,8.5);
		
		\draw [orange,->,>=stealth] (11.5,3) -- (12.5,3);
		\draw [orange,->,>=stealth] (10.8,3.5) -- (10.8,4.5);
		\draw [orange,->,>=stealth] (11.2,4.5) -- (11.2,3.5);
		
		\draw [orange,->,>=stealth] (11.5,5) -- (12.5,5);
		\draw [orange,->,>=stealth] (10.8,5.5) -- (10.8,6.5);
		\draw [orange,->,>=stealth] (11.2,6.5) -- (11.2,5.5);
		
		\draw [orange,->,>=stealth] (11.5,7) -- (12.5,7);
		\draw [orange,->,>=stealth] (10.8,7.5) -- (10.8,8.5);
		\draw [orange,->,>=stealth] (11.2,8.5) -- (11.2,7.5);
		
		\draw [orange,->,>=stealth] (11.5,9) -- (12.5,9);
		\draw [orange,->,>=stealth] (10.5,9) -- (9.5,9);
		\draw [orange,->,>=stealth] (10.8,9.5) -- (10.8,10.5);
		\draw [orange,->,>=stealth] (11.2,10.5) -- (11.2,9.5);
		
		\draw [orange,->,>=stealth] (11.5,11) -- (12.5,11);
		\draw [orange,->,>=stealth] (10.5,11) -- (9.5,11);
		\draw [orange,->,>=stealth] (11,11.5) -- (11,12.5);
		
		\node at (2.6,2) {\textcolor{red}{{\large+0.012}}}; 
		\node at (4,1.2) {\textcolor{blue}{{\large-0.012}}}; 
		
		\node at (4.6,2) {\textcolor{red}{{\large+0.007}}}; 
		\node at (6,1.2) {\textcolor{blue}{{\large-0.018}}}; 
		
		\node at (2.6,4) {\textcolor{red}{{\large+0.020}}}; 
		\node at (4,3.2) {\textcolor{blue}{{\large-0.009}}}; 
		
		\node at (0.6,6) {\textcolor{blue}{{\large-0.006}}}; 
		
		\node at (6.6,2) {\textcolor{red}{{\large+0.002}}}; 
		\node at (8,1.2) {\textcolor{blue}{{\large-0.018}}}; 
		
		\node at (4.6,4) {\textcolor{red}{{\large+0.012}}}; 
		\node at (6,3.2) {\textcolor{blue}{{\large-0.014}}}; 
		
		\node at (2.6,6) {\textcolor{red}{{\large+0.036}}}; 
		\node at (4,5.2) {\textcolor{blue}{{\large-0.010}}}; 
		\node at (2,4.8) {\textcolor{blue}{{\large-0.007}}}; 
		
		\node at (0.6,8) {\textcolor{blue}{{\large-0.044}}}; 
		
		\node at (8.6,2) {\textcolor{blue}{{\large-0.003}}}; 
		\node at (10,1.2) {\textcolor{blue}{{\large-0.013}}}; 
		
		\node at (6.6,4) {\textcolor{red}{{\large+0.004}}}; 
		\node at (8,3.2) {\textcolor{blue}{{\large-0.015}}}; 
		
		\node at (4.6,6) {\textcolor{red}{{\large+0.015}}}; 
		\node at (6,5.2) {\textcolor{blue}{{\large-0.013}}}; 
		
		\node at (2,6.8) {\textcolor{blue}{{\large-0.043}}}; 
		\node at (4,7.2) {\textcolor{red}{{\large+0.075}}}; 
		
		\node at (0.5,10) {\textcolor{blue}{{\large-0.040}}}; 
		
		\node at (10.6,2) {\textcolor{blue}{{\large-0.012}}}; 
		
		\node at (8.6,4) {\textcolor{blue}{{\large-0.006}}}; 
		\node at (10,3.2) {\textcolor{blue}{{\large-0.011}}}; 
		
		\node at (6.6,6) {\textcolor{red}{{\large+0.008}}}; 
		\node at (8,5.2) {\textcolor{blue}{{\large-0.016}}}; 
		
		\node at (6,7.2) {\textcolor{red}{{\large+0.088}}}; 
		
		\node at (0,11) {\textcolor{blue}{{\large-0.115}}}; 
		\node at (2,11) {\textcolor{blue}{{\large-0.115}}}; 
		\node at (1,12) {\textcolor{blue}{{\large-0.015}}}; 
		\node at (1.7,10) {\textcolor{blue}{{\large-0.115}}}; 
		
		\node at (10.6,4) {\textcolor{blue}{{\large-0.020}}}; 
		
		\node at (8.6,6) {\textcolor{blue}{{\large-0.005}}}; 
		\node at (10,5.2) {\textcolor{blue}{{\large-0.015}}}; 
		
		\node at (8,7.2) {\textcolor{red}{{\large+0.080}}}; 
		
		\node at (5.4,8) {\textcolor{red}{{\large+0.009}}}; 
		\node at (6,9.2) {\textcolor{blue}{{\large-0.004}}}; 
		
		\node at (10.6,6) {\textcolor{blue}{{\large-0.032}}}; 
		
		\node at (10,7.2) {\textcolor{red}{{\large+0.068}}}; 
		
		\node at (7.4,8) {\textcolor{blue}{{\large-0.008}}}; 
		
		\node at (5.4,10) {\textcolor{red}{{\large+0.005}}}; 
		\node at (6,11.2) {\textcolor{blue}{{\large-0.005}}}; 
		
		\node at (10.6,8) {\textcolor{red}{{\large+0.032}}}; 
		
		\node at (7.4,10) {\textcolor{blue}{{\large-0.004}}}; 
		
		\node at (10.6,10) {\textcolor{red}{{\large+0.029}}}; 
		
		\node at (11,12) {\textcolor{red}{{\large+0.262}}}; 
		\end{tikzpicture}
		}
		\caption{Grid world with two terminal states $G_1$ and $G_2$.}
		\label{fig:grid-world-2}
	\end{subfigure}
	\caption{Poisoning TCE in grid-world tasks.}
	\label{fig:attack-MDP}
\end{figure}
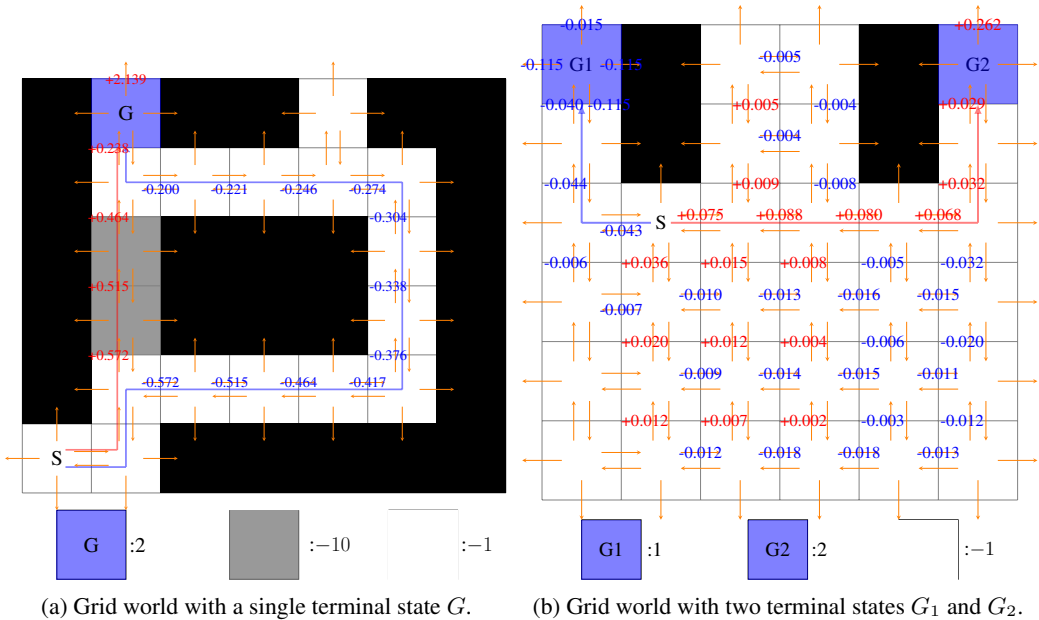

\textbf{Experiment 2.}
As another example, we consider the grid world tasks in~\cite{cakmak2012algorithmic}. In particular, we focus on  two tasks shown in figure~\ref{fig:grid-world-1} and~\ref{fig:grid-world-2}. 
In figure~\ref{fig:grid-world-1}, the agent starts from S and aims to arrive at the terminal cell G. The black regions are walls, thus the agent can only choose to go through the white or gray regions. The agent can take four actions in every state: go left, right, up or down, and stays if the action takes it into the wall. 
Reaching a gray, white, or the terminal state results in rewards $-10$, $-1$, 2, respectively. 
After the agent arrives at the terminal state G, it will stay there forever and always receive reward 0 regardless of the following actions.
The original optimal policy is to follow the blue trajectory. 
The attacker's goal is to force the agent to follow the red trajectory.
Correspondingly, we set the target actions for those states on the red trajectory as along the trajectory.
We set the target actions for the remaining states to be the same as the original optimal policy learned on clean data. 

The clean training data contains a single item for every state-action pair.
We run the attack with $\epsilon=0.1$ and $\alpha=2$. 
Our attack is successful: with the poisoned data, TCE generates a policy that produces the red trajectory in Figure~\ref{fig:grid-world-1}, which is the desired behavior. 
The attack cost is $\|\bv r-\bv r^0\|_2\approx 2.64$, which is small compared to $\|\bv r^0\|_2=21.61$. 
In Figure~\ref{fig:grid-world-1}, we show the poisoning on rewards. Each state-action pair is denoted by an orange arrow. The value tagged to each arrow is the modification to that reward, where red value means the reward is increased and blue means decreased. An arrow without any tagged value means the corresponding reward is not changed by attack.
Note that rewards along the red trajectory are increased, while those along the blue trajectory are reduced, resulting in the red trajectory being preferred by the agent. 
Furthermore, rewards closer to the starting state S suffer larger poisoning since they contribute more to the $Q$ values.
For the large attack +2.139 happening at terminal state, we provide an explanation in appendix~\ref{sparse-attack}.

\textbf{Experiment 3.}
In Figure~\ref{fig:grid-world-2} there are two terminal states G1 and G2 with reward 1 and 2, respectively. 
The agent starts from S. 
Although G2 is more profitable, the path is longer and each step has a $-1$ reward.
Therefore, the original optimal policy is the blue trajectory to G1.
The attacker's target policy is to force the agent along the red trajectory to G2. 
We set the target actions for states as in experiment 2.
The clean training data contains a single item for every state-action pair. 
We run our attack with $\epsilon=0.1$ and $\alpha=2$. Again, after the attack, TCE on the poisoned dataset produces the red trajectory in figure~\ref{fig:grid-world-2}, with attack cost $\|\bv r-\bv r^0\|_2\approx0.38$, compared to $\|\bv r^0\|_2=11.09$. The reward poisoning follows a similar pattern to experiment 2. 

\subsection{Policy Poisoning Attack on LQR Victim}
\begin{figure}[h]
	\begin{subfigure}[t]{.4\textwidth}
		\centering
		\includegraphics[width=0.95\textwidth, height=0.81\textwidth]{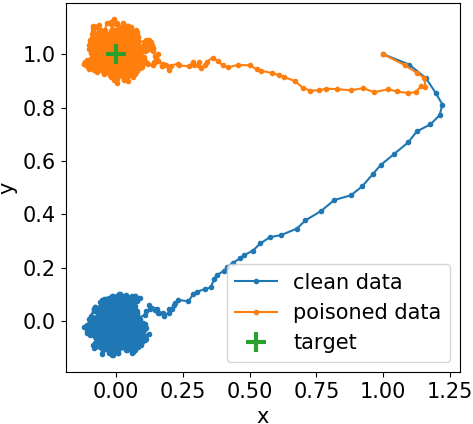}
		\caption{Clean and poisoned vehicle trajectory.}
		\label{fig:traj-LQR}
	\end{subfigure}%
	\begin{subfigure}[t]{.4\textwidth}
		\centering
		\includegraphics[width=0.95\textwidth, height=0.81\textwidth]{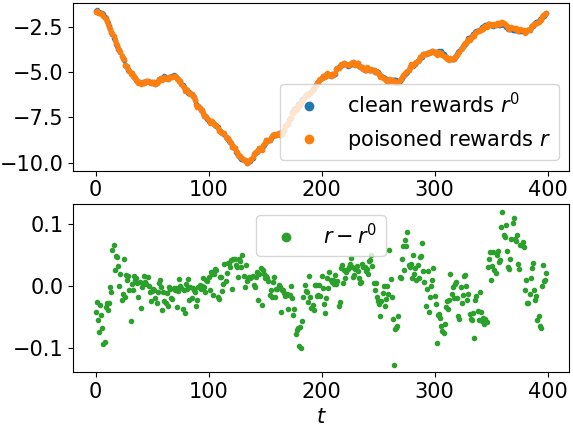}
		\caption{Clean and poisoned rewards.}
		\label{fig:poisoned_rewards}
	\end{subfigure}%
	\caption{Poisoning a vehicle running LQR in 4D state space.}
	\label{fig:poison-LQR}
\end{figure}

\textbf{Experiment 4.}
We now demonstrate our attack on LQR. We consider a linear dynamical system that approximately models a vehicle.
The state of the vehicle consists of its 2D position and 2D velocity: 
 $s_t=(x_t, y_t, v_t^x, v_t^y)\in \R^4$.
The control signal at time $t$ is the force $a_t\in \R^2$ which will be applied on the vehicle for $h$ seconds. We assume there is a friction parameter $\eta$ such that the friction force is $-\eta v_t$.
 Let $m$ be the mass of the vehicle. 
Given small enough $h$, the transition matrices can be approximated by~\eqref{eq:linear_dynamics}
where
\begin{equation}\label{exp:transition_matrix}
A=\left[\begin{array}{cccc}{1} & {0} & {h} & {0} \\ {0} & {1} & {0} & {h} \\ {0} & {0} & {1-h \eta / m} & {0} \\ {0} & {0} & {0} & {1-h \eta / m}\end{array}\right], 
B=\left[ \begin{array}{cc}{0} & {0} \\ {0} & {0} \\ {h / m} & {0} \\ {0} & {h / m}\end{array}\right].
\end{equation}
In this example, we let $h=0.1$, $m=1$, $\eta=0.5$, and $w_t\sim \mathcal{N}(0,\sigma^2I)$ with $\sigma=0.01$. The vehicle starts from initial position $(1,1)$ with velocity $(1,-0.5)$, i.e., $s_0=(1,1,1,-0.5)$. 
The true loss function is $L(s,a)=\frac{1}{2}s^\top Qs+a^\top Ra$ with $Q=I$ and $R=0.1I$ (i.e., $Q=I, R=0.1I, q=0, c=0$ in~\eqref{LQR:loss}). 
Throughout the experiment, we let $\gamma=0.9$ for solving the optimal control policy in~\eqref{LQR}. With the true dynamics and loss function, the computed optimal control policy is
\begin{equation}
K^*=\left[ \begin{array}{cccc}{-1.32} & {0} & {-2.39} & {0} \\ {0} & {-1.32} & {0} & {-2.39}\end{array}\right],
k^*=\left[ \begin{array}{cc}{0} & {0} \end{array}\right],
\end{equation}
which will drive the vehicle to the origin.

The batch LQR learner estimates the dynamics and the loss function from a batch training data. 
To produce the training data, we let the vehicle start from state $s_0$ and simulate its trajectory with a random control policy. Specifically, in each time step, we uniformly sample a control signal $a_t$ in a unit sphere. The vehicle then takes action $a_t$ to transit from current state $s_t$ to the next state $s_{t+1}$, and receives a reward $r_t=-L(s_t,a_t)$. This gives us one training item $(s_t,a_t,r_t,s_{t+1})$. We simulate a total of 400 time steps to obtain a batch data that contains 400 items, on which the learner estimates the dynamics and the loss function. With the learner's estimate, the computed clean optimal policy is
\begin{equation}
\hat K^0=\left[ \begin{array}{cccc}{-1.31} & {1.00\mathrm{e}{-2}} & {-2.41} & {2.03\mathrm{e}{-3}} \\ {-1.97\mathrm{e}{-2}} & {-1.35} & {-1.14\mathrm{e}{-2}} & {-2.42}\end{array}\right],
\hat k^0=\left[ \begin{array}{cc}{-4.88\mathrm{e}{-5}} & {4.95\mathrm{e}{-6}} \end{array}\right].
\end{equation}
The clean optimal policy differs slightly from the true optimal policy due to the inaccuracy of the learner's estimate.
The attacker has a target policy $(K^\dagger, k^\dagger)$ that can drive the vehicle close to its target destination $(x^\dagger,y^\dagger)=(0,1)$ with terminal velocity $(0,0)$, which can be represented as a target state $s^\dagger=(0,1,0,0)$. 
To ensure feasibility, we assume that the attacker starts with the loss function $\frac{1}{2}(s-s^\dagger)^\top Q(s-s^\dagger)+a^\top Ra$ where $Q=I, R=0.1I$.  Due to the offset this corresponds to setting $Q=I, R=0.1I, q=-s^\dagger, c=\frac{1}{2}{s^\dagger}^\top Qs^\dagger=0.5$ in~\eqref{LQR:loss}. The attacker then solves the Riccati equation with its own loss function and the learner's estimates $\hat A$ and $\hat B$ to arrive at the target policy
\begin{equation}
K^\dagger=\left[ \begin{array}{cccc}{-1.31} &  {9.99\mathrm{e}{-3}} & {-2.41} & {2.02\mathrm{e}{-3}} \\ {-1.97\mathrm{e}{-2}}  & {-1.35} & {-1.14\mathrm{e}{-2}}  & {-2.42}\end{array}\right],
k^\dagger=\left[ \begin{array}{cc}{-0.01} & {1.35} \end{array}\right].
\end{equation}

We run our attack~\eqref{attack:objective-LQR}-\eqref{attack-LQR-c6} with $\alpha=2$ and $\epsilon=0.01$ in~\eqref{attack-LQR-c5}. Figure~\ref{fig:poison-LQR} shows the result of our attack. In Figure~\ref{fig:traj-LQR}, we plot the trajectory of the vehicle with policy learned on clean data and poisoned data respectively. 
Our attack successfully forces LQR into a policy that drives the vehicle close to the target destination. 
The wiggle on the trajectory is due to the noise $w_t$ of the dynamical system. 
On the poisoned data, the LQR victim learns the policy
\begin{equation}
\hat K=\left[ \begin{array}{cccc}{-1.31} & {9.99\mathrm{e}{-3}} & {-2.41} & {2.02\mathrm{e}{-3}} \\ {-1.97\mathrm{e}{-2}} & {-1.35} & {-1.14\mathrm{e}{-2}} & {-2.42}\end{array}\right],
\hat k=\left[ \begin{array}{cc}{-0.01} & {1.35} \end{array}\right],
\end{equation}
which matches exactly the target policy $K^\dagger$, $k^\dagger$. 
In Figure~\ref{fig:poisoned_rewards}, we show the poisoning on rewards. 
Our attack leads to very small modification on each reward, thus the attack is efficient. The total attack cost over all 400 items is only $\|\bv r-\bv r^0\|_2=0.73$, which is tiny small compared to $\|\bv r^0\|_2=112.94$. The results here demonstrate that our attack can dramatically change the behavior of LQR by only slightly modifying the rewards in the dataset.

Finally, for both attacks on TCE and LQR, we note that by setting the attack cost norm $\alpha=1$ in~\eqref{eq:objective}, the attacker is able to obtain a \emph{sparse} attack, meaning that only a small fraction of the batch data needs to be poisoned.  Such sparse attacks have profound implications in adversarial machine learning as they can be easier to carry out and harder to detect.  We show detailed results in appendix~\ref{sparse-attack}.

\section{Conclusion}
We presented a policy poisoning framework against batch reinforcement learning and control.
We showed the attack problem can be formulated as convex optimization.
We provided theoretical analysis on attack feasibility and cost. 
Experiments show the attack can force the learner into an attacker-chosen target policy while incurring only a small attack cost.

\textbf{Acknowledgments.} This work is supported in part by NSF 1545481, 1561512, 1623605, 1704117, 1836978 and the MADLab AF Center of Excellence FA9550-18-1-0166.

\bibliography{attack_q_learning}
\bibliographystyle{plainnat}

\input{appendix}

\end{document}

%% file: def.tex
\newcommand{\Q}{\mathcal{Q}}
\newcommand{\RR}{\mathcal{R}}
\newcommand{\bv}[1]{\mathbf{#1}}
\newcommand{\A}{\mathcal{A}}
\newtheorem{theorem}{Theorem}
\newtheorem{lemma}[theorem]{Lemma}
\newtheorem{proposition}[theorem]{Proposition}
\newtheorem{corollary}[theorem]{Corollary}
\newtheorem{definition}{Definition}

\newcommand{\thmref}[1]{Theorem~\ref{#1}}
\newcommand{\lemref}[1]{Lemma~\ref{#1}}

\newcommand{\I}{\mathcal{I}}

\renewcommand{\S}{\mathcal{S}}
\renewcommand\ind[1]{\ensuremath{\mathds{1}\left[#1\right]}}

\newcommand{\interior}[1]{%
  {\kern0pt#1}^{\mathrm{o}}%
}

\definecolor{C0}{HTML}{1F77B4}
\definecolor{C1}{HTML}{FF7F0E}
\definecolor{C2}{HTML}{2ca02c}
\definecolor{C3}{HTML}{d62728}
\definecolor{C4}{HTML}{9467bd}
\definecolor{C5}{HTML}{8c564b}
\definecolor{C6}{HTML}{e377c2}
\definecolor{C7}{HTML}{7F7F7F}
\definecolor{C8}{HTML}{bcbd22}
\definecolor{C9}{HTML}{17BECF}

%% file: previous.tex
\section{Related Work}
Of particular interest is the work on \emph{test-time  attacks} against RL~\cite{huang2017adversarial}.
Unlike policy poisoning, there the RL agent carries out an already-learned and fixed policy $\pi$ to e.g. play the Pong Game.
The attacker perturbs pixels in a game board image, which is part of the state $s$.
This essentially changes the RL agent's perceived state into some $s'$.
The RL agent then chooses the action $a':=\pi(s')$ (e.g. move down) which may differ from $a := \pi(s)$ (e.g. move up).
The attacker's goal is to force some specific $a'$ on the RL agent.
Note $\pi$ itself stays the same through the attack.
In contrast, ours is a data-poisoning attack which happens at training time and aims to change $\pi$.

Data-poisoning attacks were previously limited to supervised learning victims, either in batch mode~\cite{biggio2012poisoning,xiao2015feature,li2016data,mei2015using} or online mode~\cite{wang2018data,zhang2019online}. 
Recently data-poisoning attacks have been extended to multi-armed bandit victims~\cite{jun2018adversarial,ma2018data,liu2019data}, but not yet to RL victims.

There are two related but distinct concepts in RL research.
One concept is reward shaping~\cite{ng1999policy,asmuth2008potential,devlin2012dynamic,wiewiora2003potential} which also modifies rewards to affect an RL agent. However, the goal of reward shaping is fundamentally different from ours. Reward shaping aims to speed up convergence to the \emph{same} optimal policy as without shaping.
Note the differences in both the target (same vs. different policies) and the optimality measure (speed to converge vs. magnitude of reward change).

The other concept is teaching IRL~\cite{cakmak2012algorithmic,brown2019machine,kamalaruban19}. 
Teaching and attacking are mathematically equivalent.  However, the main difference to our work is the victim.
They require an IRL agent, which is a specialized algorithm that estimates a reward function from demonstrations of (state, action) trajectories alone (i.e. no reward given).
In contrast, our attacks target more prevalent RL agents and are thus potentially more applicable.
Due to the difference in the input to IRL vs. RL victims, our attack framework is completely different. 

%% file: definition.tex
\section{Preliminaries}
\label{prelim}
A Markov Decision Process (MDP) is defined as a tuple $(\S, \A, P, R, \gamma)$, 
where $\S$ is the state space, 
$\A$ is the action space, 
$P:\S\times\A \to \Delta_\S$ is the transition kernel where $\Delta_\S$ denotes the space of probability distributions on $\S$,
$R:\S\times\A \to \mathbb{R}$ is the reward function,
and $\gamma\in [0,1)$ is a discounting factor. 
We define a policy $\pi:\S \to \A$ as a function that maps a state to an action. 
We denote the $Q$ function of a policy $\pi$ as 
$Q^{\pi}(s,a) = \mathbb{E}[ \sum_{\tau=0}^{\infty} \gamma^\tau R(s_\tau,a_\tau)  \mid  s_0 = s, a_0 = a, \pi]$, where the expectation is over the randomness in both transitions and rewards. The $Q$ function that corresponds to the optimal policy can be characterized by the following Bellman optimality equation:
\begin{align}
Q^*(s,a) = R(s, a ) + \gamma \sum_{s'\in \S} P(s'|s,a) \max_{a'\in\A} Q^*(s',a'),\label{eq:Bellman}
\end{align}  and the optimal policy is defined as $\pi^*(s) \in \argmax_{a\in \A} Q^*(s,a)$.

We focus on RL victims who perform batch reinforcement learning.
A training item is a tuple $(s,a,r,s^\prime)\in \S\times\A\times\R\times \S$, where $s$ is the current state, $a$ is the action taken, $r$ is the received reward, and $s^\prime$ is the next state.
A training set is a batch of $T$ training items denoted by $D=(s_t,a_t,r_t,s_t^\prime)_{t=0:T-1}$. 
Given training set $D$, a model-based learner performs learning in two steps:

\textbf{Step 1}. The learner estimates an MDP $\hat M=(\S, \A, \hat P, \hat R, \gamma)$ from $D$. In particular, we assume the learner uses maximum likelihood estimate for
the transition kernel $\hat P : \S\times \A\mapsto \Delta_\S$
	\begin{eqnarray}\label{MLE}
	\hat P &\in& \argmax_{P} \sum_{t=0}^{T-1} \log P(s_t^\prime|s_t,a_t),
	\end{eqnarray}
 and least-squares estimate for the reward function $\hat R : \S\times\A\mapsto \R$
	\begin{eqnarray}\label{MMSE}
	\hat R &=& \argmin_{R} \sum_{t=0}^{T-1} (r_t - R(s_t,a_t))^2.
	\end{eqnarray}
Note that we do not require~\eqref{MLE} to have a unique maximizer $\hat P$. When multiple maximizers exist, we assume the learner arbitrarily picks one of them as the estimate. 
We assume the minimizer $\hat R$ is always unique. We will discuss the conditions to guarantee the uniqueness of $\hat R$ for two learners later.

\textbf{Step 2}. The learner finds the optimal policy $\hat\pi$ that maximizes the expected discounted cumulative reward on the estimated environment $\hat M$, i.e.,
	\begin{eqnarray}\label{optimal_policy}
	\hat\pi \in \argmax_{\pi : \S \mapsto \A} \mathbb{E}_{\hat{P}} \sum_{\tau=0}^\infty \gamma^\tau\hat R(s_\tau,\pi(s_\tau)),
	\end{eqnarray}
where $s_0$ is a specified or random initial state.
Note that there could be multiple optimal policies, thus we use $\in$ in~\eqref{optimal_policy}.
Later we will specialize~\eqref{optimal_policy} to two specific victim learners: the tabular certainty equivalence learner (TCE) and the certainty-equivalent linear quadratic regulator (LQR).

%% file: appendix.tex
\newpage
\appendix
{\Large \textbf{Supplementary Material}}
\section{Proof of Proposition~\ref{feasibility}}
The proof of feasibility relies on the following result, which states that there is a bijection mapping between reward space and value function space.

\begin{proposition}\label{rQ:bijection}
Given an MDP with transition probability function $P$ and discounting factor $\gamma\in[0,1)$, let $\RR = \{R: \S\times \A\mapsto \R\}$ denote the set of all possible reward functions, and let $\Q = \{Q: \S\times \A\mapsto \R\}$ denote the set of all possible Q tables. 
Then, there exists a bijection mapping between $\RR$ and $\Q$, induced by Bellman optimality equation.
\end{proposition}
\begin{proof}
$\Rightarrow$ Given any reward function $R(s,a)\in\RR$, define the Bellman operator as
\begin{equation}
H_R(Q)(s,a)=R(s, a)+\gamma\sum_{s^\prime}P(s^\prime\mid s, a)\max _{a^\prime} Q(s^\prime, a^\prime).
\end{equation}
Since $\gamma<1$, $H_R(Q)$ is a contraction mapping, i.e., $\|H_R(Q_1)-H_R(Q_2)\|_\infty\le \gamma\|Q_1-Q_2\|_\infty$, $\forall Q_1,Q_2\in\Q$. Then by Banach Fixed Point Theorem, there is a unique $Q\in\Q$ that satisfies $Q=H_R(Q)$, which is the $Q$ that $R$ maps to.

$\Leftarrow$ Given any $Q\in\Q$, one can define the corresponding $R\in\RR$ by
\begin{eqnarray}
R(s,a) = Q(s,a) - \gamma\sum_{s^\prime}P(s^\prime\mid s, a)\max _{a^\prime} Q(s^\prime, a^\prime).
\end{eqnarray}
Thus the mapping is one-to-one.
\end{proof}

\feasibility*
\begin{proof}
For any target policy $\pi^\dagger:\S\mapsto \A$, we construct the following $Q$:
\begin{equation}\label{eq:constructedQ}
Q(s,a)=\left\{
\begin{aligned}
&\epsilon &&\forall s\in \S, a=\pi^\dagger(s),\\
&0,&& \text{ otherwise.}
\end{aligned}
\right.
\end{equation}
The $Q$ values in~\eqref{eq:constructedQ} satisfy the constraint~\eqref{attack:c3}. Note that we construct the $Q$ values so that for all $s\in\S$, $\max_{a}Q(s,a)=\epsilon$. By proposition~\ref{rQ:bijection}, the corresponding reward function induced by Bellman optimality equation is
\begin{equation}\label{eq:constructedRhat}
\hat R(s,a)=\left\{
\begin{aligned}
&(1-\gamma)\epsilon &&\forall s\in \S,a=\pi^\dagger(s),\\
&-\gamma\epsilon,&& \text{ otherwise.}
\end{aligned}
\right.
\end{equation}
Then one can let $r_t= \hat R(s_t,a_t)$ so that $\bv r=(r_0,...,r_{T-1})$, $\hat R$ in~\eqref{eq:constructedRhat}, together with $Q$ in~\eqref{eq:constructedQ} is a feasible solution to~\eqref{attack:formulation}-\eqref{attack:c3}.
\end{proof}

\section{Proof of~\thmref{TCE:bound-cost}}
The proof of~\thmref{TCE:bound-cost} relies on a few lemmas. We first prove the following result, which shows that given two vectors that have equal element summation, the vector whose elements are smoother will have smaller $\ell_\alpha$ norm for any $\alpha\ge 1$.
This result is used later to prove~\lemref{optimal-attack-form}.
\begin{lemma}{}{}\label{lem:allequal}
Let $x,y\in\R^T$ be two vectors. Let $\mathcal{I}\subset \{0,1,...,T-1\}$ be a subset of indexes such that
\begin{equation}
i).\quad x_i=\frac{1}{|\I|}\sum_{j\in \I}y_j, \forall i\in \I, \quad\quad ii).\quad x_i=y_i, \forall i\neq \I. 
\end{equation}
Then for any $\alpha\ge 1$, we have $\|x\|_\alpha\le \|y\|_\alpha$.
\end{lemma}
\begin{proof}
Note that the conditions $i)$ and $ii)$ suggest the summation of elements in $x$ and $y$ are equal, and only elements in $\mathcal{I}$ differ for the two vectors. However, the elements in $\mathcal{I}$ of $x$ are smoother than that of $y$, thus $x$ has smaller norm. To prove the result, we consider three cases separately.

Case 1: $\alpha=1$. Then we have
\begin{equation}
\|x\|_\alpha-\|y\|_\alpha=\sum_{i}|x_i|-\sum_{j}|y_j|=\sum_{i\in \I}|x_i|-\sum_{j\in \I}|y_j|=|\sum_{j\in \I}y_j|-\sum_{j\in \I}|y_j|\le0.
\end{equation}

Case 2: $1< \alpha<\infty$. We show $\|x\|_\alpha^\alpha\le\|y\|_\alpha^\alpha$. Note that
\begin{equation}
\begin{aligned}\label{eq:allequal-case2}
\|x\|_\alpha^\alpha-\|y\|_\alpha^\alpha&=\sum_{i}|x_i|^\alpha-\sum_{j}|y_j|^\alpha=\sum_{i\in \mathcal{I}}|x_i|^\alpha-\sum_{j\in \mathcal{I}}|y_j|^\alpha\\
&=\frac{1}{|\I|^{\alpha-1}}|\sum_{j\in \I}y_j|^\alpha-\sum_{j\in \mathcal{I}}|y_j|^\alpha\le\frac{1}{|\I|^{\alpha-1}}(\sum_{j\in \I}|y_j|)^\alpha-\sum_{j\in \mathcal{I}}|y_j|^\alpha.
\end{aligned}
\end{equation}
Let $\beta=\frac{\alpha}{\alpha-1}$. By Holder's inequality, we have
\begin{equation}\label{eq:allequal-case2holder}
\sum_{j\in \I}|y_j|\le (\sum_{j\in \I}|y_j|^\alpha)^\frac{1}{\alpha}(\sum_{j\in \I}1^\beta)^\frac{1}{\beta}= (\sum_{j\in \I}|y_j|^\alpha)^\frac{1}{\alpha}|\I|^{1-\frac{1}{\alpha}}.
\end{equation}
Plugging~\eqref{eq:allequal-case2holder} into~\eqref{eq:allequal-case2}, we have
\begin{equation}
\|x\|_\alpha^\alpha-\|y\|_\alpha^\alpha\le\frac{1}{|\I|^{\alpha-1}}(\sum_{j\in \I}|y_j|^\alpha)|\I|^{\alpha-1}-\sum_{j\in \mathcal{I}}|y_j|^\alpha= 0.
\end{equation}

Case 3: $\alpha=\infty$. We have
\begin{equation}
\begin{aligned}
\|x\|_\alpha&=\max_i|x_i|=\max\{\frac{1}{|\I|}|\sum_{j\in \I}y_j|,\max_{i\notin \I}|x_i|\}\le\max\{\frac{1}{|\I|}\sum_{j\in \I}|y_j|,\max_{i\notin \I}|x_i|\}\\
&\le \max\{\max_{j\in \I}|y_j|,\max_{i\notin \I}|x_i|\}=\max\{\max_{j\in \I}|y_j|,\max_{j\notin \I}|y_j|\}=\max_j |y_j|=\|y\|_\alpha.
\end{aligned}
\end{equation}
Therefore $\forall \alpha\ge1$, we have $\|x\|_\alpha\le \|y\|_\alpha$.
\end{proof}

Next we prove~\lemref{optimal-attack-form}, which shows that one possible optimal attack solution to~\eqref{attack:formulation}-\eqref{attack:c3} takes the following form: shift all the clean rewards in $T_{s,a}$ by the same amount $\psi(s,a)$. Here $\psi(s,a)$ is a function of state $s$ and action $a$. That means, rewards belonging to different $T_{s,a}$ might be shifted a different amount, but those corresponding to the same $(s,a)$ pair will be identically shifted. 
\begin{lemma}{}{}\label{optimal-attack-form}
There exists a function $\psi(s,a)$ such that $r_t=r_t^0+\psi(s_t,a_t)$, together with some $\hat R$ and $Q$,  is an optimal solution to our attack problem~\eqref{attack:formulation}-\eqref{attack:c3}.
\end{lemma}
We point out that although there exists an optimal attack taking the above form, it is not necessarily the only optimal solution. 
However, all those optimal solutions must have exactly the same objective value (attack cost), thus it suffices to consider the solution in~\lemref{optimal-attack-form}.
\begin{proof}
Let $\bv r^*=(r_0^*,...,r_{T-1}^*)$, $\hat R^*$ and $Q^*$ be any optimal solution to~\eqref{attack:formulation}-\eqref{attack:c3}. Fix a particular state-action pair $(s,a)$, we have
\begin{equation}
\hat R^*(s,a)=\frac{1}{|T_{s,a}|}\sum_{t\in T_{s,a}}r^*_t.
\end{equation}
Let $\hat R^0(s,a)=\frac{1}{|T_{s,a}|}\sum_{t\in T_{s,a}}r^0_t$ be the reward function for the $(s,a)$ pair estimated from clean data $\bv r^0$. We then define a different poisoned reward vector $\bv r^\prime=(r_0^\prime,...,r_{T-1}^\prime)$, where
\begin{equation}
r^\prime_t=\left\{
\begin{aligned}
&r_t^0+\hat R^*(s,a)-\hat R^0(s,a),&t\in T_{s,a},\\
&r^*_t,& t\notin T_{s,a}.\\
\end{aligned}
\right.
\end{equation}
Now we show $\bv r^\prime$, $\hat R^*$ and $Q^*$ is another optimal solution to~\eqref{attack:formulation}-\eqref{attack:c3}. 
We first verify that $\bv r^\prime$, $\hat R^*$, and $Q^*$ satisfy constraints~\eqref{attack:c1}-\eqref{attack:c3}. To verify~\eqref{attack:c1}, we only need to check $\hat R^*(s,a)=\frac{1}{|T_{s,a}|}\sum_{t\in T_{s,a}}r^\prime_t$, since $\bv r^\prime$ and $\bv r^*$ only differ on those rewards in $T_{s,a}$. We have
\begin{equation}
\begin{aligned}
\frac{1}{|T_{s,a}|}\sum_{t\in T_{s,a}}r^\prime_t&=\frac{1}{|T_{s,a}|}\sum_{t\in T_{s,a}} \left(r_t^0+\hat R^*(s,a)-\hat R^0(s,a)\right)\\
&=\hat R^0(s,a)+\hat R^*(s,a)-\hat R^0(s,a)=\hat R^*(s,a),
\end{aligned}
\end{equation}
Thus $\bv r^\prime$ and $\hat R^*$ satisfy constraint~\eqref{attack:c1}. 
$\hat R^*$ and $Q^*$ obviously satisfy constraints~\eqref{attack:c2} and~\eqref{attack:c3} because $\bv r^*$, $\hat R^*$ and $Q^*$ is an optimal solution.

Let $\delta^\prime=\bv r^\prime-\bv r^0$ and $\delta^*=\bv r^*-\bv r^0$, then one can easily show that $\delta^\prime$ and $\delta^*$ satisfy the conditions in~\lemref{lem:allequal} with $\I=T_{s,a}$. 
Therefore by~\lemref{lem:allequal}, we have
\begin{equation}
\|\bv r^\prime-\bv r^0\|_\alpha=\|\delta^\prime\|_\alpha\le \|\delta^*\|_\alpha=\|\bv r^*-\bv r^0\|_\alpha.
\end{equation}
But note that by our assumption, $\bv r^*$ is an optimal solution, thus $\|\bv r^*-\bv r^0\|_\alpha\le \|\bv r^\prime-\bv r^0\|_\alpha$, which gives $\|\bv r^\prime-\bv r^0\|_\alpha= \|\bv r^*-\bv r^0\|_\alpha$.
This suggests $\bv r^\prime$, $\hat R^*$, and $Q^*$ is another optimal solution. Compared to $\bv r^*$, $\bv r^\prime$ differs in that $r^\prime_t-r^0_t$ now becomes identical for all $t\in T_{s,a}$ for a particular $(s,a)$ pair.
Reusing the above argument iteratively, one can make $r^\prime_t-r^0_t$ identical for all $t\in T_{s,a}$ for all $(s,a)$ pairs, while guaranteeing the solution is still optimal. Therefore, we have
\begin{equation}
 r^\prime_t=r_t^0+\hat R^*(s,a)-\hat R^0(s,a), \forall t\in T_{s,a},\forall s,a,
\end{equation}
together with $\hat R^*$ and $Q^*$ is an optimal solution to~\eqref{attack:formulation}-\eqref{attack:c3}. 
Let $\psi(s,a)=\hat R^*(s,a)-\hat R^0(s,a)$ conclude the proof.
\end{proof}

Finally, \lemref{rQ:sensitivity} provides a sensitive analysis on the value function $Q$ as the reward function changes. 
\begin{lemma}{}{}\label{rQ:sensitivity}
Let $\hat M=(\S,\A,\hat P,\hat R^\prime,\gamma)$ and $\hat M^0=(\S,\A,\hat P,\hat R^0,\gamma)$ be two MDPs, where only the reward function differs. Let $Q^\prime$ and $Q^0$ be action values satisfying the Bellman optimality equation on $\hat M$ and $\hat M^0$ respectively, then
\begin{equation}
(1-\gamma)\|Q^\prime-Q^0\|_\infty\le \|\hat R-\hat R^0\|_\infty \le (1+\gamma)\|Q^\prime-Q^0\|_\infty.
\end{equation}
\end{lemma}
\begin{proof}
Define the Bellman operator as
\begin{equation}
H_{\hat R}(Q)(s,a)=\hat R(s, a)+\gamma\sum_{s^\prime} \hat P(s^\prime\mid s, a) \max _{a^\prime} Q(s^\prime, a^\prime).
\end{equation}
From now on we suppress variables $s$ and $a$ for convenience. 
Note that due to the Bellman optimality, we have $H_{\hat R^0}(Q^0)=Q^0$ and $H_{\hat R^\prime}(Q^\prime)=Q^\prime$, thus
\begin{equation}
\begin{aligned}
\|Q^\prime-Q^0\|_\infty&=\|H_{\hat R^\prime}(Q^\prime)-H_{\hat R^0}(Q^0)\|_\infty\\
&=\|H_{\hat R^\prime}(Q^\prime)-H_{\hat R^\prime}(Q^0)+H_{\hat R^\prime}(Q^0)-H_{\hat R^0}(Q^0)\|_\infty\\
&\le\|H_{\hat R^\prime}(Q^\prime)-H_{\hat R^\prime}(Q^0)\|_\infty+\|H_{\hat R^\prime}(Q^0)-H_{\hat R^0}(Q^0)\|_\infty\\
&\le \gamma \|Q^\prime-Q^0\|_\infty+\|H_{\hat R^\prime}(Q^0)-H_{\hat R^0}(Q^0)\|_\infty \text{ (by contraction of $H_{\hat R^\prime}(\cdot)$)}\\
&=\gamma \|Q^\prime-Q^0\|_\infty+\|\hat R^\prime-\hat R^0\|_\infty \text{ (by $H_{\hat R^\prime}(Q^0)-H_{\hat R^0}(Q^0) =\hat R^\prime-\hat R^0$)}\\
\end{aligned}
\end{equation}
Rearranging we have $(1-\gamma)\|Q^\prime-Q^0\|_\infty\le \|\hat R^\prime-\hat R^0\|_\infty$. Similarly we have
\begin{equation}
\begin{aligned}
\|Q^\prime-Q^0\|_\infty&=\|H_{\hat R^\prime}(Q^\prime)-H_{\hat R^0}(Q^0)\|_\infty\\
&=\|H_{\hat R^\prime}(Q^0)-H_{\hat R^0}(Q^0)+H_{\hat R^\prime}(Q^\prime)-H_{\hat R^\prime}(Q^0)\|_\infty\\
&\ge \|H_{\hat R^\prime}(Q^0)-H_{\hat R^0}(Q^0)\|_\infty-\|H_{\hat R^\prime}(Q^\prime)-H_{\hat R^\prime}(Q^0)\|_\infty\\
&\ge\|H_{\hat R^\prime}(Q^0)-H_{\hat R^0}(Q^0)\|_\infty -\gamma \|Q^\prime-Q^0\|_\infty\\
&=\|\hat R^\prime-\hat R^0\|_\infty-\gamma \|Q^\prime-Q^0\|_\infty
\end{aligned}
\end{equation}
Rearranging we have $\|\hat R^\prime-\hat R^0\|_\infty\le (1+\gamma)\|Q^\prime-Q^0\|_\infty$, concluding the proof.
\end{proof}

Now we are ready to prove our main result.
\TCEboundcost*
\begin{proof}
We construct the following value function $Q^\prime$.
\begin{equation}
Q^\prime(s,a)=\left\{
\begin{aligned}
&Q^0(s,a)+\frac{\Delta(\epsilon)}{2}, &&\forall s\in \S, a=\pi^\dagger(s),\\
&Q^0(s,a)-\frac{\Delta(\epsilon)}{2}, &&\forall s\in \S, \forall a\neq \pi^\dagger(s).\\
\end{aligned}
\right.
\end{equation}
Note that $\forall s\in\S$ and $\forall a\neq\pi^\dagger(s)$, we have
\begin{equation}
\begin{aligned}
\Delta(\epsilon)&=\max_{s^\prime\in\S}[\max_{a^\prime\neq \pi^\dagger(s^\prime)} Q^0(s^\prime,a^\prime)-Q^0(s^\prime,\pi^\dagger(s^\prime))+\epsilon]_+\\
&\ge \max_{a^\prime\neq \pi^\dagger(s)} Q^0(s,a^\prime)-Q^0(s,\pi^\dagger(s))+\epsilon\ge  Q^0(s,a)-Q^0(s,\pi^\dagger(s))+\epsilon,
\end{aligned}
\end{equation}
which leads to
\begin{equation}
Q^0(s,a)-Q^0(s,\pi^\dagger(s)) -\Delta(\epsilon)\le -\epsilon,
\end{equation}
thus we have $\forall  s\in\S$ and $\forall a\neq \pi^\dagger(s)$,
\begin{equation}
\begin{aligned}
Q^\prime(s,\pi^\dagger(s))&= Q^0(s,\pi^\dagger(s))+\frac{\Delta(\epsilon)}{2}\\
&=Q^0(s,a)-[Q^0(s,a)-Q^0(s,\pi^\dagger(s))-\Delta(\epsilon)]-\frac{\Delta(\epsilon)}{2}\\
&\ge Q^0(s,a)+\epsilon-\frac{\Delta(\epsilon)}{2} =Q^\prime(s,a)+\epsilon.
\end{aligned}
\end{equation}
Therefore $Q^\prime$ satisfies the constraint~\eqref{attack:c3}. 
By proposition~\ref{rQ:bijection}, there exists a unique function $R^\prime$ such that $Q^\prime$ satisfies the Bellman optimality equation of MDP $\hat M^\prime=(\S,\A, \hat P, R^\prime,\gamma)$. 
We then construct the following reward vector $\bv r^\prime=(r_0^\prime,...,r_{T-1}^\prime)$ such that $\forall (s,a)$ and $\forall t\in T_{s,a}$, $r_t^\prime=r_t^0+R^\prime(s,a)-\hat R^0(s,a)$, where $\hat R^0(s,a)$ is the reward function estimated from $\bv r^0$.
The reward function estimated on $\bv r^\prime$ is then
\begin{equation}
\begin{aligned}
\hat R^\prime(s,a)&=\frac{1}{|T_{s,a}|}\sum_{t\in T_{s,a}}r_t^\prime=\frac{1}{|T_{s,a}|}\sum_{t\in T_{s,a}}\left(r_t^0+R^\prime(s,a)-\hat R^0(s,a)\right)\\
&= \hat R^0(s,a)+R^\prime(s,a)-\hat R^0(s,a)=R^\prime(s,a).
\end{aligned}
\end{equation} 
Thus $\bv r^\prime$, $\hat R^\prime$ and $Q^\prime$ is a feasible solution to~\eqref{attack:formulation}-\eqref{attack:c3}.
Now we analyze the attack cost for $\bv r^\prime$, which gives us a natural upper bound on the attack cost of the optimal solution $\bv r^*$. 
Note that $Q^\prime$ and $Q^0$ satisfy the Bellman optimality equation for reward function $\hat R^\prime$ and $\hat R^0$ respectively, and
\begin{equation}
\|Q^\prime-Q^0\|_\infty=\frac{\Delta(\epsilon)}{2},
\end{equation}
thus by~\lemref{rQ:sensitivity}, we have $\forall t$,
\begin{equation}\label{eq:upper}
\begin{aligned}
|r^\prime_t-r^0_t|&=|\hat R^\prime(s_t,a_t)-\hat R^0(s_t,a_t)|\le \max_{s,a}|\hat R^\prime(s,a)-\hat R^0(s,a)|=\|\hat R^\prime-\hat R^0\|_\infty\\
&\le(1+\gamma)\|Q^\prime-Q^0\|_\infty=\frac{1}{2}(1+\gamma)\Delta(\epsilon).
\end{aligned}
\end{equation}
Therefore, we have
\begin{equation}
\|\bv r^*-\bv r^0\|_\alpha\le \|\bv r^\prime-\bv r^0\|_\alpha=(\sum_{t=0}^{T-1}|r^\prime_t-r^0_t|^\alpha)^{\frac{1}{\alpha}}\le\frac{1}{2}(1+\gamma)\Delta(\epsilon) T^\frac{1}{\alpha}.
\end{equation}

Now we prove the lower bound. We consider two cases separately.

Case 1: $\Delta(\epsilon)=0$. We must have $Q^0(s,\pi^\dagger(s))\ge Q^0(s,a)+\epsilon$, $\forall s\in\S,\forall a\neq \pi^\dagger(s)$. In this case no attack is needed and therefore the optimal solution is $\bv r^*=\bv r^0$. The lower bound holds trivially.

Case 2: $\Delta(\epsilon)>0$. Let $s^\prime$ and $a^\prime$ ($a^\prime\neq \pi^\dagger(s^\prime)$) be a state-action pair such that 
\begin{equation}
\Delta(\epsilon)=Q^0(s^\prime,a^\prime)-Q^0(s^\prime,\pi^\dagger(s^\prime))+\epsilon.
\end{equation}
Let $\bv r^*$, $\hat R^*$ and $Q^*$ be an optimal solution to~\eqref{attack:formulation}-\eqref{attack:c3} that takes the form in~\lemref{optimal-attack-form}, i.e.,
\begin{equation}
 r^*_t=r_t^0+\hat R^*(s,a)-\hat R^0(s,a), \forall t\in T_{s,a},\forall s,a.
 \end{equation}
Constraint~\eqref{attack:c3} ensures that $Q^*(s^\prime,\pi^\dagger(s^\prime))\ge Q^*(s^\prime,a^\prime)+\epsilon$, in which case either one of the following two conditions must hold:
\begin{equation}
i).\quad Q^*(s^\prime,\pi^\dagger(s^\prime))-Q^0(s^\prime,\pi^\dagger(s^\prime))\ge \frac{\Delta(\epsilon)}{2}, \quad\quad ii).\quad Q^0(s^\prime,a^\prime)-Q^*(s^\prime,a^\prime)\ge \frac{\Delta(\epsilon)}{2},
\end{equation}
since otherwise we have
\begin{equation}
\begin{aligned}
&Q^*(s^\prime,\pi^\dagger(s^\prime))<Q^0(s^\prime,\pi^\dagger(s^\prime))+\frac{\Delta(\epsilon)}{2}=Q^0(s^\prime,\pi^\dagger(s^\prime))+\frac{1}{2}[Q^0(s^\prime,a^\prime)-Q^0(s^\prime,\pi^\dagger(s^\prime))+\epsilon]\\
&=\frac{1}{2}Q^0(s^\prime,a^\prime)+\frac{1}{2}Q^0(s^\prime,\pi^\dagger(s^\prime))+\frac{\epsilon}{2}=Q^0(s^\prime,a^\prime)-\frac{1}{2}[Q^0(s^\prime,a^\prime)-Q^0(s^\prime,\pi^\dagger(s^\prime))+\epsilon]+\epsilon\\
&=Q^0(s^\prime,a^\prime)-\frac{\Delta(\epsilon)}{2}+\epsilon< Q^*(s^\prime,a^\prime)+ \epsilon.\\
\end{aligned}
\end{equation}
Next note that if either $i)$ or $ii)$  holds, we have $\|Q^*-Q^0\|_\infty\ge \frac{\Delta(\epsilon)}{2}$. By~\lemref{rQ:sensitivity}, we have
\begin{equation}\label{eq:lower}
\max_{s,a}|\hat R^*(s,a)-\hat R^0(s,a)|=\|\hat R^*-\hat R^0\|_\infty\ge(1-\gamma)\|Q^*-Q^0\|_\infty\ge \frac{1}{2}(1-\gamma)\Delta(\epsilon).
\end{equation}
Let $s^*,a^*\in\arg\max_{s,a}|\hat R^*(s,a)-\hat R^0(s,a)|$, then we have
\begin{equation}
|\hat R^*(s^*,a^*)-\hat R^0(s^*,a^*)|\ge \frac{1}{2}(1-\gamma)\Delta(\epsilon).
\end{equation}
Therefore, we have
\begin{equation}
\begin{aligned}
\|\bv r^*-\bv r^0\|_\alpha^\alpha&=\sum_{t=0}^{T-1}|r_t^*-r_t^0|^\alpha=\sum_{s,a}\sum_{t\in T_{s,a}} |r_t^*-r_t^0|^\alpha\ge\sum_{t\in T_{s^*,a^*}}|r_t^*-r_t^0|^\alpha \\
&=\sum_{t\in T_{s^*,a^*}}|\hat R^*(s^*,a^*)-\hat R^0(s^*,a^*)|^\alpha\ge\left(\frac{1}{2}(1-\gamma)\Delta(\epsilon)\right)^\alpha|T_{s^*,a^*}| \\
&\ge\left(\frac{1}{2}(1-\gamma)\Delta(\epsilon)\right)^\alpha\min_{s,a} {|T_{s,a}|}.
\end{aligned}
\end{equation}
Therefore $\|\bv r^*-\bv r^0\|_\alpha\ge \frac{1}{2}(1-\gamma)\Delta(\epsilon)\left(\min_{s,a} {|T_{s,a}|}\right)^\frac{1}{\alpha}$. 

We finally point out that while an optimal solution $\bv r^*$ may not necessarily take the form in~\lemref{optimal-attack-form}, it suffices to bound the cost of an optimal attack which indeed takes this form (as we did in the proof) since all optimal attacks have exactly the same objective value.
\end{proof}

\section{Convex Surrogate for LQR Attack Optimization}
\label{convex:surrogate}
By pulling the positive semi-definite constraints on $Q$ and $R$ out of the lower level optimization~\eqref{attack-LQR-c5}, one can turn the original attack optimization~\eqref{attack:objective-LQR}-\eqref{attack-LQR-c6} into the following surrogate optimization:
\begin{eqnarray}\label{attack:LQR-single}
  \min_{\bv r,  \hat Q, \hat R, \hat q, \hat c, X, x}  &&\|\bv r-\bv r_0\|_\alpha\\ 
  \text { s.t. }
   &&-\gamma\left(\hat R+\gamma \hat B^{\top} X \hat B\right)^{-1} \hat B^{\top} X \hat A=K^\dagger,\label{attack:LQR-single-c1}\\
   && -\gamma\left(\hat R+\gamma \hat B^{\top} X \hat B\right)^{-1} \hat B^{\top} x = k^\dagger ,\label{attack:LQR-single-c2}\\
    && X=\gamma \hat A^{\top} X \hat A-\gamma^{2} \hat A^{\top} X \hat B\left(\hat R+\gamma \hat B^{\top} X \hat B\right)^{-1} \hat B^{\top} X \hat A+\hat Q\label{attack:LQR-single-c3}\\
     && x =   \hat q + \gamma(\hat A+\hat BK^\dagger)^\top x \label{attack:LQR-single-c4}\\
     &&(\hat{Q}, \hat{R},\hat q,\hat c) = \argmin \sum_{t=0}^{T-1}\left\|\frac{1}{2}s_t^\top Qs_t+q^\top s_t+a_t^\top Ra_t+c+r_t\right\|_{2}^{2}\label{attack:LQR-single-KKT2}\\
    &&\hat Q\succeq 0, \hat R\succeq \epsilon I, X\succeq 0\label{attack-LQR-single-c9}.
\end{eqnarray}
The feasible set of~\eqref{attack:LQR-single}-\eqref{attack-LQR-single-c9} is a subset of the original problem, thus the surrogate attack optimization is a more stringent  formulation than the original attack optimization, that is, successfully solving the surrogate optimization gives us a (potentially) sub-optimal solution to the original problem. To see why the surrogate optimization is more stringent, we illustrate with a much simpler example as below. A formal proof is straight forward, thus we omit it here. The original problem is~\eqref{bilevel-example}-\eqref{bilevel-example-c1}. The feasible set for $\hat a$ is a singleton set $\{0\}$, and the optimal objective value is 0.\begin{eqnarray}\label{bilevel-example}
  \min_{\hat a}  && 0\\ 
  \text { s.t. }
   && \hat a = \argmin_{a\ge0} (a+3)^2,\label{bilevel-example-c1}
  \end{eqnarray}
 Once we pull the constraint out of the lower-level optimization~\eqref{bilevel-example-c1}, we end up with a surrogate optimization~\eqref{bilevel-example-surrogate}-\eqref{bilevel-example-surrogate-c2}. Note that~\eqref{bilevel-example-surrogate-c1} requires $\hat a=-3$, which does not satisfy~\eqref{bilevel-example-surrogate-c2}. Therefore the feasible set of the surrogate optimization is $\emptyset$, meaning it is more stringent than~\eqref{bilevel-example}-\eqref{bilevel-example-c1}.
  \begin{eqnarray}\label{bilevel-example-surrogate}
  \min_{\hat a}  && 0\\ 
  \text { s.t. }
   && \hat a = \argmin (a+3)^2,\label{bilevel-example-surrogate-c1}\\
   && \hat a\ge 0\label{bilevel-example-surrogate-c2}
  \end{eqnarray}

Back to our attack optimization~\eqref{attack:LQR-single}-\eqref{attack-LQR-single-c9}, this surrogate attack optimization comes with the advantage of being convex, thus can be solved to global optimality.
\begin{restatable}{proposition}{attackLQRconvexity}
\label{attackLQR:convexity}
The surrogate attack optimization~\eqref{attack:LQR-single}-\eqref{attack-LQR-single-c9} is convex.
\end{restatable}

\begin{proof}
First note that the sub-level optimization~\eqref{attack:LQR-single-KKT2} is itself a convex problem, thus is equivalent to the corresponding KKT condition. We write out the KKT condition of~\eqref{attack:LQR-single-KKT2} to derive an explicit form of our attack formulation as below:
\begin{eqnarray}\label{attack:LQR-single-app}
  \min_{\bv r,  \hat Q, \hat R, \hat q, \hat c, X, x}  &&\|\bv r-\bv r_0\|_\alpha\\ 
  \text { s.t. }
   &&-\gamma\left(\hat R+\gamma \hat B^{\top} X \hat B\right)^{-1} \hat B^{\top} X \hat A=K^\dagger,\label{attack:LQR-single-c1-app}\\
   && -\gamma\left(\hat R+\gamma \hat B^{\top} X \hat B\right)^{-1} \hat B^{\top} x = k^\dagger ,\label{attack:LQR-single-c2-app}\\
    && X=\gamma \hat A^{\top} X \hat A-\gamma^{2} \hat A^{\top} X \hat B\left(\hat R+\gamma \hat B^{\top} X \hat B\right)^{-1} \hat B^{\top} X \hat A+\hat Q\label{attack:LQR-single-c3-app}\\
     && x =   \hat q + \gamma(\hat A+\hat BK^\dagger)^\top x \label{attack:LQR-single-c4-app}\\
  &&\sum_{t=0}^{T-1} (\frac{1}{2}s_t^\top \hat Qs_t+\hat q^\top s_t+a_t^\top \hat Ra_t+\hat c+r_t)s_ts_t^\top=0,\label{attack:LQR-single-c5-app}\\
  &&\sum_{t=0}^{T-1} (\frac{1}{2}s_t^\top \hat Qs_t+\hat q^\top s_t+a_t^\top \hat Ra_t+\hat c+r_t)a_ta_t^\top=0,\label{attack:LQR-single-c6-app}\\
  &&\sum_{t=0}^{T-1} (\frac{1}{2}s_t^\top \hat Qs_t+\hat q^\top s_t+a_t^\top \hat Ra_t+\hat c+r_t)s_t=0,\label{attack:LQR-single-c7-app}\\
  &&\sum_{t=0}^{T-1} (\frac{1}{2}s_t^\top \hat Qs_t+\hat q^\top s_t+a_t^\top \hat Ra_t+\hat c+r_t)=0,\label{attack:LQR-single-c8-app}\\
    &&\hat Q\succeq 0, \hat R\succeq \epsilon I, X\succeq 0\label{attack-LQR-single-c9-app}.
\end{eqnarray}
The objective is obviously convex.~\eqref{attack:LQR-single-c1-app}-\eqref{attack:LQR-single-c3-app} are equivalent to
\begin{equation}
-\gamma \hat B^{\top} X \hat A=\left(\hat R+\gamma \hat B^{\top} X \hat B\right)K^\dagger.
\end{equation}
\begin{equation}
-\gamma \hat B^{\top} x=\left(\hat R+\gamma \hat B^{\top} X \hat B\right)k^\dagger.
\end{equation}
\begin{equation}
X=\gamma \hat A^\top X (\hat A+\hat BK^\dagger)+\hat Q,
\end{equation}
Note that these three equality constraints are all linear in $X$, $\hat R$, $x$, and $\hat Q$.
\eqref{attack:LQR-single-c4-app} is linear in $x$ and $\hat q$.~\eqref{attack:LQR-single-c5-app}-\eqref{attack:LQR-single-c8-app} are also linear in $\hat Q$, $\hat R$, $\hat q$, $\hat c$ and $\bv r$.
Finally,~\eqref{attack-LQR-single-c9-app} contains convex constraints on $\hat Q$, $\hat R$, and $X$. Given all above, the attack problem is convex.
\end{proof}

Next we analyze the feasibility of the surrogate attack optimization.
\begin{restatable}{proposition}{attackLQRfeasible}
\label{attackLQR:feasible}
Let $\hat A$, $\hat B$ be the learner's estimated transition kernel. Let 
\begin{equation}
L^\dagger(s,a)=\frac{1}{2}s^\top Q^\dagger s+(q^\dagger)^\top s+a^\top R^\dagger a+c^\dagger
\end{equation}
be the attacker-defined loss function. Assume $R^\dagger\succeq \epsilon I$.  If the target policy $K^\dagger$, $k^\dagger$ is the optimal control policy induced by the LQR with transition kernel $\hat A$, $\hat B$, and loss function $L^\dagger(s,a)$, then the surrogate attack optimization~\eqref{attack:LQR-single}-\eqref{attack-LQR-single-c9} is feasible. Furthermore, the optimal solution can be achieved.
\end{restatable}
\begin{proof}
To prove feasibility, it suffices to construct a feasible solution to optimization~\eqref{attack:LQR-single}-\eqref{attack-LQR-single-c9}. Let 
\begin{equation}
r_t=\frac{1}{2}s_t^\top Q^\dagger s_t+{q^\dagger}^\top s_t+a_t^\top R^\dagger a_t+c^\dagger
\end{equation}
and $\bv r$ be the vector whose $t$-th element is $r_t$. We next show that $\bv r$, $Q^\dagger$, $R^\dagger$, $q^\dagger$, $c^\dagger$, together with some $X$ and $x$ is a feasible solution. Note that since $K^\dagger$, $k^\dagger$ is induced by the LQR with transition kernel $\hat A$, $\hat B$ and cost function $L^\dagger(s,a)$, constraints~\eqref{attack:LQR-single-c1}-\eqref{attack:LQR-single-c4} must be satisfied with some $X$ and $x$. The poisoned reward vector $\bv r$ obviously satisfies~\eqref{attack:LQR-single-KKT2} since it is constructed exactly as the minimizer. By our assumption, $R^\dagger\succeq \epsilon I$, thus~\eqref{attack-LQR-single-c9} is satisfied. Therefore, $\bv r$, $Q^\dagger$, $R^\dagger$, $q^\dagger$, $c^\dagger$, together with the corresponding $X$, $x$ is a feasible solution, and the optimization~\eqref{attack:LQR-single}-\eqref{attack-LQR-single-c9} is feasible. Furthermore, since the feasible set is closed, the optimal solution can be achieved.
\end{proof}

\section{Conditions for The LQR Learner to Have Unique Estimate}\label{LQR:condition_uniqueness}
The LQR learner estimates the cost function by
\begin{equation}\label{LQR:estimate-reward-app}
(\hat{Q}, \hat{R},\hat q,\hat c)= \argmin _{(Q\succeq 0, R\succeq \epsilon I,q,c)}  \frac{1}{2}\sum_{t=0}^{T-1}\left\|\frac{1}{2}s_t^\top Qs_t+q^\top s_t+a_t^\top Ra_t+c+r_t\right\|_{2}^{2}.
\end{equation}
We want to find a condition that guarantees the uniqueness of the solution.

Let $\psi\in\R^T$ be a vector, whose $t$-th element is
\begin{equation}\label{psi}
\psi_t=\frac{1}{2}s_t^\top Qs_t+q^\top s_t+a_t^\top Ra_t+c, 0\le t\le T-1.
\end{equation}

Note that we can view $\psi$ as a function of $D$, $Q$, $R$, $q$, and $c$, thus we can also denote $\psi(D,Q,R,q,c)$.
Define $\Psi(D)=\{\psi (D,Q,R,q,c)\mid Q\succeq 0, R\succeq \epsilon I,q,c\}$, i.e., all possible vectors that are achievable with form~\eqref{psi} if we vary $Q$, $R$, $q$ and $c$ subject to positive semi-definite constraints on $Q$ and $R$. We can prove that $\Psi$ is a closed convex set.

\begin{proposition}
$\forall D$, $\Psi(D)=\{\psi (D,Q,R,q,c)\mid Q\succeq 0, R\succeq \epsilon I,q,c\}$ is a closed convex set.
\end{proposition}
\begin{proof}
Let $\psi_1, \psi_2\in\Psi(D)$. We use $\psi_{i,t}$ to denote the $t$-th element of vector $\psi_i$. Then we have
\begin{equation}
\psi_{1,t}=\frac{1}{2}s_t^\top Q_1s_t+q_1^\top s_t+a_t^\top R_1a_t+c_1
\end{equation}
for some $Q_1\succeq 0$, $R_1\succeq \epsilon I$, $q_1$ and $c_1$, and
\begin{equation}
\psi_{2,t}=\frac{1}{2}s_t^\top Q_2s_t+q_2^\top s_t+a_t^\top R_2a_t+c_2
\end{equation}
for some $Q_2\succeq 0$, $R_2\succeq\epsilon I$, $q_2$ and $c_2$. $\forall k\in[0,1]$, let $\psi_3 = k\psi_1+(1-k)\psi_2$. Then the $t$-th element of $\psi_3$ is
\begin{equation}
\begin{aligned}
\psi_{3,t}=&\frac{1}{2}s_t^\top [kQ_1+(1-k)Q_2]s_t+[kq_1+(1-k)q_2]^\top s_t\\
&+a_t^\top [kR_1+(1-k)R_2]a_t+kc_1+(1-k)c_2
\end{aligned}
\end{equation}
Since $kQ_1+(1-k)Q_2\succeq 0$ and $kR_1+(1-k)R_2\succeq \epsilon I$, $\psi_3\in \Psi(D)$, concluding the proof.
\end{proof}

The optimization~\eqref{LQR:estimate-reward-app} is intrinsically a least-squares problem with positive semi-definite constraints on $Q$ and $R$, and is equivalent to solving the following linear equation:
\begin{equation}\label{LQR:estimate-reward-app-equivalence}
\frac{1}{2}s_t^\top \hat Qs_t+\hat q^\top s_t+a_t^\top \hat Ra_t+\hat c=\psi_t^*, \forall t,
\end{equation}
where $\psi^*=\argmin_{\psi\in\Psi(D)}\|\psi+\bv r\|^2_2$ is the projection of the negative reward vector $-\bv r$ onto the set $\Psi(D)$. 
The solution to~\eqref{LQR:estimate-reward-app-equivalence} is unique if and only if the following two conditions both hold 
\begin{itemize}
\item [$i).$] The projection $\psi^*$ is unique.
\item[$ii).$] ~\eqref{LQR:estimate-reward-app-equivalence} has a unique solution for $\psi^*$.
\end{itemize}
Condition $i)$ is satisfied because $\Psi(D)$ is convex, and any projection (in $\ell_2$ norm) onto a convex set exists and is always unique (see Hilbert Projection Theorem). We next analyze when condition $ii)$ holds.~\eqref{LQR:estimate-reward-app-equivalence} is a linear function in $\hat Q$, $\hat R$, $\hat q$, and $\hat c$, thus one can vectorize $\hat Q$ and $\hat R$ to obtain a problem in the form of linear regression. Then the uniqueness is guaranteed if and only if the design matrix has full column rank. Specifically, let $\hat Q\in\R^{n\times n}$, $\hat R\in \R^{m\times m}$, and $\hat q\in\R^n$. Let $s_{t,i}$ and $a_{t,i}$ denote the $i$-th element of $s_t$ and $a_t$ respectively. Define
$$
\bv A=
\arraycolsep=3pt
\left[\begin{array}{ccccc|ccccc|c|c}
{\frac{s_{0,1}^2}{2}} &\ldots & {\frac{s_{0,i}s_{0,j}}{2}} &\ldots & {\frac{s_{0,n}^2}{2}} & {a_{0,1}^2} &\ldots & {a_{0,i}a_{0,j}} &\ldots & {a_{0,m}^2} & {s_0^\top} & 1\\
{\frac{s_{1,1}^2}{2}} &\ldots & {\frac{s_{1,i}s_{1,j}}{2}} &\ldots & {\frac{s_{1,n}^2}{2}} & {a_{1,1}^2} &\ldots & {a_{1,i}a_{2,j}} &\ldots & {a_{1,m}^2} & {s_1^\top} & 1\\
\vdots& & \vdots & & \vdots & \vdots  & & \vdots&& \vdots  &\vdots & \vdots \\
{\frac{s_{t,1}^2}{2}} &\ldots & {\frac{s_{t,i}s_{t,j}}{2}} &\ldots & {\frac{s_{t,n}^2}{2}} & {a_{t,1}^2} &\ldots & {a_{t,i}a_{t,j}} &\ldots & {a_{t,m}^2} & {s_t^\top} & 1\\
\vdots& & \vdots & & \vdots & \vdots  & & \vdots&& \vdots  &\vdots &  \vdots \\
{\frac{s_{T-1,1}^2}{2}} &\ldots & {\frac{s_{T-1,i}s_{T-1,j}}{2}} &\ldots & {\frac{s_{T-1,n}^2}{2}} & {a_{T-1,1}^2} &\ldots & {a_{T-1,i}a_{T-1,j}} &\ldots & {a_{T-1,m}^2} & {s_{T-1}^\top} & 1\\
\end{array}\right],
$$
$$
\bv x^\top = \left[\begin{array}{ccccc|ccccc|ccccc|c}
{\hat Q_{11}} & \ldots & {\hat Q_{ij}}  &  \ldots & {\hat Q_{nn}} & {\hat R_{11}} & \ldots & {\hat R_{ij}}  & \ldots & {\hat R_{mm}} & {\hat q_{1}} & \ldots & {\hat q_{i}} &\ldots & {\hat q_{n}} & {\hat c}\\
\end{array}\right],
$$
then~\eqref{LQR:estimate-reward-app-equivalence} is equivalent to $\bv A\bv x=\psi^*$, where $\bv x$ contains the vectorized variables $\hat Q$, $\hat R$, $\hat q$ and $\hat c$. $\bv A\bv x=\psi^*$ has a unique solution if and only if $\bv A$ has full column rank.

\section{Sparse Attacks on TCE and LQR}\label{sparse-attack}
In this section, we present experimental details for both TCE and LQR victims when the attacker uses $\ell_1$ norm to measure the attack cost, i.e. $\alpha=1$. The other experimental parameters are set exactly the same as in the main text.

We first show the result for MDP experiment 2 with $\alpha=1$, see Figure~\ref{fig:grid-world-1-complete-sparse}. The attack cost is $\|\bv r -\bv r^0\|_1=3.27$, which is small compared to $\|\bv r^0\|_1=105$. We note that the reward poisoning is extremely sparse: only the reward corresponding to action ``go up'' at the terminal state $G$ is increased by 3.27, and all other rewards remain unchanged. To explain this attack, first note that we set the target action for the terminal state to ``go up'', thus the corresponding reward must be increased. Next note that after the attack, the terminal state becomes a sweet spot, where the agent can keep taking action ``go up'' to gain large amount of discounted future reward. However, such future reward is discounted more if the agent reaches the terminal state via a longer path. Therefore, the agent will choose to go along the red trajectory to get into the terminal state earlier, though at a price of two discounted $-10$ rewards.

\begin{figure}[]
		\centering
		\begin{tikzpicture}
		\draw[step=2cm,black,thin,opacity=0.5] (0,0) grid (14,12);
		\fill[black] (0,2) rectangle (2,12);
		\fill[black] (4,0) rectangle (14,2);
		\fill[black] (4,4) rectangle (10,8);
		\fill[black] (4,10) rectangle (8,12);
		\fill[black] (12,0) rectangle (14,12);
		\fill[black] (10,10) rectangle (12,12);
		\fill[gray,opacity=0.8] (2,4) rectangle (4,8);
		\fill[blue,opacity=0.5] (2,10) rectangle (4,12);
		\node at (1,1) {S};
		\node at (3,11) {G};
		
		\draw [orange,->,>=stealth] (0.5,1) -- (-0.5,1);
		\draw [orange,->,>=stealth] (1,0.5) -- (1,-0.5);
		\draw [orange,->,>=stealth] (1.5,1.2) -- (2.5,1.2);
		\draw [orange,->,>=stealth] (2.5,0.8) -- (1.5,0.8);
		\draw [orange,->,>=stealth] (1,1.5) -- (1,2.5);
		
		\draw [orange,->,>=stealth] (3,0.5) -- (3,-0.5);
		\draw [orange,->,>=stealth] (3.2,2.5) -- (3.2,1.5);
		\draw [orange,->,>=stealth] (3.5,1) -- (4.5,1);
		\draw [orange,->,>=stealth] (2.8,1.5) -- (2.8,2.5);
		
		\draw [orange,->,>=stealth] (3.2,4.5) -- (3.2,3.5);
		\draw [orange,->,>=stealth] (3.5,3.2) -- (4.5,3.2);
		\draw [orange,->,>=stealth] (4.5,2.8) -- (3.5,2.8);
		\draw [orange,->,>=stealth] (2.8,3.5) -- (2.8,4.5);
		\draw [orange,->,>=stealth] (2.5,3) -- (1.5,3);
		
		\draw [orange,->,>=stealth] (5.5,3.2) -- (6.5,3.2);
		\draw [orange,->,>=stealth] (6.5,2.8) -- (5.5,2.8);
		\draw [orange,->,>=stealth] (5,3.5) -- (5,4.5);
		\draw [orange,->,>=stealth] (5,2.5) -- (5,1.5);
		
		\draw [orange,->,>=stealth] (7.5,3.2) -- (8.5,3.2);
		\draw [orange,->,>=stealth] (8.5,2.8) -- (7.5,2.8);
		\draw [orange,->,>=stealth] (7,3.5) -- (7,4.5);
		\draw [orange,->,>=stealth] (7,2.5) -- (7,1.5);
		
		\draw [orange,->,>=stealth] (9.5,3.2) -- (10.5,3.2);
		\draw [orange,->,>=stealth] (10.5,2.8) -- (9.5,2.8);
		\draw [orange,->,>=stealth] (9,3.5) -- (9,4.5);
		\draw [orange,->,>=stealth] (9,2.5) -- (9,1.5);
		
		\draw [orange,->,>=stealth] (11.2,4.5) -- (11.2,3.5);
		\draw [orange,->,>=stealth] (11.5,3) -- (12.5,3);
		\draw [orange,->,>=stealth] (10.8,3.5) -- (10.8,4.5);
		\draw [orange,->,>=stealth] (11,2.5) -- (11,1.5);
		
		\draw [orange,->,>=stealth] (11.2,6.5) -- (11.2,5.5);
		\draw [orange,->,>=stealth] (11.5,5) -- (12.5,5);
		\draw [orange,->,>=stealth] (10.8,5.5) -- (10.8,6.5);
		\draw [orange,->,>=stealth] (10.5,5) -- (9.5,5);
		
		\draw [orange,->,>=stealth] (11.2,8.5) -- (11.2,7.5);
		\draw [orange,->,>=stealth] (11.5,7) -- (12.5,7);
		\draw [orange,->,>=stealth] (10.8,7.5) -- (10.8,8.5);
		\draw [orange,->,>=stealth] (10.5,7) -- (9.5,7);
		
		\draw [orange,->,>=stealth] (11.5,9) -- (12.5,9);
		\draw [orange,->,>=stealth] (11,9.5) -- (11,10.5);
		\draw [orange,->,>=stealth] (9.5,9.2) -- (10.5,9.2);
		\draw [orange,->,>=stealth] (10.5,8.8) -- (9.5,8.8);
		
		\draw [orange,->,>=stealth] (8.8,9.5) -- (8.8,10.5);
		\draw [orange,->,>=stealth] (9.2,10.5) -- (9.2,9.5);
		\draw [orange,->,>=stealth] (7.5,9.2) -- (8.5,9.2);
		\draw [orange,->,>=stealth] (8.5,8.8) -- (7.5,8.8);
		\draw [orange,->,>=stealth] (9,8.5) -- (9,7.5);
		
		\draw [orange,->,>=stealth] (7,9.5) -- (7,10.5);
		\draw [orange,->,>=stealth] (5.5,9.2) -- (6.5,9.2);
		\draw [orange,->,>=stealth] (6.5,8.8) -- (5.5,8.8);
		\draw [orange,->,>=stealth] (7,8.5) -- (7,7.5);
		
		\draw [orange,->,>=stealth] (5,9.5) -- (5,10.5);
		\draw [orange,->,>=stealth] (3.5,9.2) -- (4.5,9.2);
		\draw [orange,->,>=stealth] (4.5,8.8) -- (3.5,8.8);
		\draw [orange,->,>=stealth] (5,8.5) -- (5,7.5);
		
		\draw [orange,->,>=stealth] (2.8,9.5) -- (2.8,10.5);
		\draw [orange,->,>=stealth] (3.2,10.5) -- (3.2,9.5);
		\draw [orange,->,>=stealth] (2.5,9) -- (1.5,9);
		\draw [orange,->,>=stealth] (2.8,7.5) -- (2.8,8.5);
		\draw [orange,->,>=stealth] (3.2,8.5) -- (3.2,7.5);
		
		\draw [orange,->,>=stealth] (2.8,5.5) -- (2.8,6.5);
		\draw [orange,->,>=stealth] (3.2,6.5) -- (3.2,5.5);
		\draw [orange,->,>=stealth] (2.5,7) -- (1.5,7);
		\draw [orange,->,>=stealth] (3.5,7) -- (4.5,7);
		
		\draw [orange,->,>=stealth] (2.5,5) -- (1.5,5);
		\draw [orange,->,>=stealth] (3.5,5) -- (4.5,5);
		
		\draw [orange,->,>=stealth] (2.5,11) -- (1.5,11);
		\draw [orange,->,>=stealth] (3.5,11) -- (4.5,11);
		\draw [orange,->,>=stealth] (3,11.5) -- (3,12.5);
		
		\draw [orange,->,>=stealth] (8.5,11) -- (7.5,11);
		\draw [orange,->,>=stealth] (9.5,11) -- (10.5,11);
		\draw [orange,->,>=stealth] (9,11.5) -- (9,12.5);
		\node at (3,12) {\textcolor{red}{{\normalsize+3.270}}}; 
		\end{tikzpicture}
		\caption{Sparse reward modification for MDP experiment 2.}
		\label{fig:grid-world-1-complete-sparse}
\end{figure}
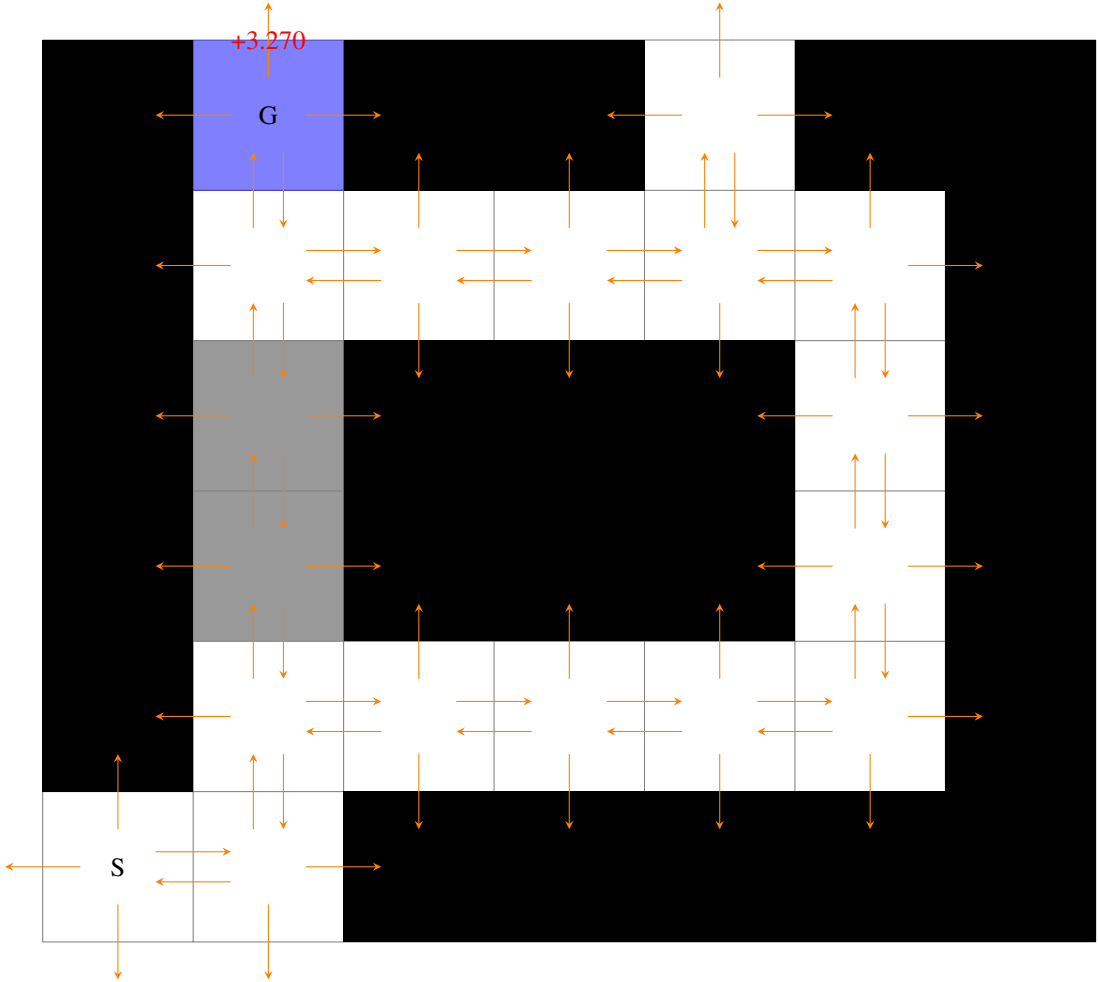

The result is similar for MDP experiment 3. The attack cost is $\|\bv r -\bv r^0\|_1=1.05$, compared to $\|\bv r^0\|_1=121$. In Figure~\ref{fig:grid-world-2-complete-sparse}, we show the reward modification for each state action pair. Again, the attack is very sparse: only rewards of 12 state-action pairs are modified out of a total of 124.

\begin{figure}[]
		\centering
		\begin{tikzpicture}
		\draw[step=2cm,black,thin,opacity=0.5] (0,0) grid (12,12);
		\fill[black] (2,8) rectangle (4,12);
		\fill[black] (8,8) rectangle (10,12);
		\node at (1,11) {G1};
		\node at (11,11) {G2};
		\fill[blue,opacity=0.5] (0,10) rectangle (2,12);
		\fill[blue,opacity=0.5] (10,10) rectangle (12,12);
		\node at (3,7) {S};
		
		\draw [orange,->,>=stealth] (0.5,1) -- (-0.5,1);
		\draw [orange,->,>=stealth] (1,0.5) -- (1,-0.5);
		\draw [orange,->,>=stealth] (1.5,1.2) -- (2.5,1.2);
		\draw [orange,->,>=stealth] (2.5,0.8) -- (1.5,0.8);
		\draw [orange,->,>=stealth] (0.8,1.5) -- (0.8,2.5);
		\draw [orange,->,>=stealth] (1.2,2.5) -- (1.2,1.5);
		
		\draw [orange,->,>=stealth] (0.5,3) -- (-0.5,3);
		\draw [orange,->,>=stealth] (1.5,3.2) -- (2.5,3.2);
		\draw [orange,->,>=stealth] (2.5,2.8) -- (1.5,2.8);
		\draw [orange,->,>=stealth] (0.8,3.5) -- (0.8,4.5);
		\draw [orange,->,>=stealth] (1.2,4.5) -- (1.2,3.5);
		
		\draw [orange,->,>=stealth] (0.5,5) -- (-0.5,5);
		\draw [orange,->,>=stealth] (1.5,5.2) -- (2.5,5.2);
		\draw [orange,->,>=stealth] (2.5,4.8) -- (1.5,4.8);
		\draw [orange,->,>=stealth] (0.8,5.5) -- (0.8,6.5);
		\draw [orange,->,>=stealth] (1.2,6.5) -- (1.2,5.5);
		
		\draw [orange,->,>=stealth] (0.5,7) -- (-0.5,7);
		\draw [orange,->,>=stealth] (1.5,7.2) -- (2.5,7.2);
		\draw [orange,->,>=stealth] (2.5,6.8) -- (1.5,6.8);
		\draw [orange,->,>=stealth] (0.8,7.5) -- (0.8,8.5);
		\draw [orange,->,>=stealth] (1.2,8.5) -- (1.2,7.5);
		
		\draw [orange,->,>=stealth] (0.5,9) -- (-0.5,9);
		\draw [orange,->,>=stealth] (1.5,9) -- (2.5,9);
		\draw [orange,->,>=stealth] (0.8,9.5) -- (0.8,10.5);
		\draw [orange,->,>=stealth] (1.2,10.5) -- (1.2,9.5);
		
		\draw [orange,->,>=stealth] (0.5,11) -- (-0.5,11);
		\draw [orange,->,>=stealth] (1.5,11) -- (2.5,11);
		\draw [orange,->,>=stealth] (1,11.5) -- (1,12.5);
		
		\draw [orange,->,>=stealth] (3,0.5) -- (3,-0.5);
		\draw [orange,->,>=stealth] (3.5,1.2) -- (4.5,1.2);
		\draw [orange,->,>=stealth] (4.5,0.8) -- (3.5,0.8);
		\draw [orange,->,>=stealth] (2.8,1.5) -- (2.8,2.5);
		\draw [orange,->,>=stealth] (3.2,2.5) -- (3.2,1.5);
		
		\draw [orange,->,>=stealth] (5,0.5) -- (5,-0.5);
		\draw [orange,->,>=stealth] (5.5,1.2) -- (6.5,1.2);
		\draw [orange,->,>=stealth] (6.5,0.8) -- (5.5,0.8);
		\draw [orange,->,>=stealth] (4.8,1.5) -- (4.8,2.5);
		\draw [orange,->,>=stealth] (5.2,2.5) -- (5.2,1.5);
		
		\draw [orange,->,>=stealth] (7,0.5) -- (7,-0.5);
		\draw [orange,->,>=stealth] (7.5,1.2) -- (8.5,1.2);
		\draw [orange,->,>=stealth] (8.5,0.8) -- (7.5,0.8);
		\draw [orange,->,>=stealth] (6.8,1.5) -- (6.8,2.5);
		\draw [orange,->,>=stealth] (7.2,2.5) -- (7.2,1.5);
		
		\draw [orange,->,>=stealth] (9,0.5) -- (9,-0.5);
		\draw [orange,->,>=stealth] (9.5,1.2) -- (10.5,1.2);
		\draw [orange,->,>=stealth] (10.5,0.8) -- (9.5,0.8);
		\draw [orange,->,>=stealth] (8.8,1.5) -- (8.8,2.5);
		\draw [orange,->,>=stealth] (9.2,2.5) -- (9.2,1.5);
		
		\draw [orange,->,>=stealth] (11,0.5) -- (11,-0.5);
		\draw [orange,->,>=stealth] (11.5,1) -- (12.5,1);
		\draw [orange,->,>=stealth] (10.8,1.5) -- (10.8,2.5);
		\draw [orange,->,>=stealth] (11.2,2.5) -- (11.2,1.5);
		
		\draw [orange,->,>=stealth] (3.5,3.2) -- (4.5,3.2);
		\draw [orange,->,>=stealth] (4.5,2.8) -- (3.5,2.8);
		\draw [orange,->,>=stealth] (2.8,3.5) -- (2.8,4.5);
		\draw [orange,->,>=stealth] (3.2,4.5) -- (3.2,3.5);
		
		\draw [orange,->,>=stealth] (3.5,5.2) -- (4.5,5.2);
		\draw [orange,->,>=stealth] (4.5,4.8) -- (3.5,4.8);
		\draw [orange,->,>=stealth] (2.8,5.5) -- (2.8,6.5);
		\draw [orange,->,>=stealth] (3.2,6.5) -- (3.2,5.5);
		
		\draw [orange,->,>=stealth] (3.5,7.2) -- (4.5,7.2);
		\draw [orange,->,>=stealth] (4.5,6.8) -- (3.5,6.8);
		\draw [orange,->,>=stealth] (3,7.5) -- (3,8.5);
		
		\draw [orange,->,>=stealth] (5.5,3.2) -- (6.5,3.2);
		\draw [orange,->,>=stealth] (6.5,2.8) -- (5.5,2.8);
		\draw [orange,->,>=stealth] (4.8,3.5) -- (4.8,4.5);
		\draw [orange,->,>=stealth] (5.2,4.5) -- (5.2,3.5);
		
		\draw [orange,->,>=stealth] (5.5,5.2) -- (6.5,5.2);
		\draw [orange,->,>=stealth] (6.5,4.8) -- (5.5,4.8);
		\draw [orange,->,>=stealth] (4.8,5.5) -- (4.8,6.5);
		\draw [orange,->,>=stealth] (5.2,6.5) -- (5.2,5.5);
		
		\draw [orange,->,>=stealth] (5.5,7.2) -- (6.5,7.2);
		\draw [orange,->,>=stealth] (6.5,6.8) -- (5.5,6.8);
		\draw [orange,->,>=stealth] (4.8,7.5) -- (4.8,8.5);
		\draw [orange,->,>=stealth] (5.2,8.5) -- (5.2,7.5);
		
		\draw [orange,->,>=stealth] (5.5,9.2) -- (6.5,9.2);
		\draw [orange,->,>=stealth] (6.5,8.8) -- (5.5,8.8);
		\draw [orange,->,>=stealth] (4.8,9.5) -- (4.8,10.5);
		\draw [orange,->,>=stealth] (5.2,10.5) -- (5.2,9.5);
		\draw [orange,->,>=stealth] (4.5,9) -- (3.5,9);
		
		\draw [orange,->,>=stealth] (5.5,11.2) -- (6.5,11.2);
		\draw [orange,->,>=stealth] (6.5,10.8) -- (5.5,10.8);
		\draw [orange,->,>=stealth] (5,11.5) -- (5,12.5);
		\draw [orange,->,>=stealth] (4.5,11) -- (3.5,11);
		
		\draw [orange,->,>=stealth] (7.5,3.2) -- (8.5,3.2);
		\draw [orange,->,>=stealth] (8.5,2.8) -- (7.5,2.8);
		\draw [orange,->,>=stealth] (6.8,3.5) -- (6.8,4.5);
		\draw [orange,->,>=stealth] (7.2,4.5) -- (7.2,3.5);
		
		\draw [orange,->,>=stealth] (7.5,5.2) -- (8.5,5.2);
		\draw [orange,->,>=stealth] (8.5,4.8) -- (7.5,4.8);
		\draw [orange,->,>=stealth] (6.8,5.5) -- (6.8,6.5);
		\draw [orange,->,>=stealth] (7.2,6.5) -- (7.2,5.5);
		
		\draw [orange,->,>=stealth] (7.5,7.2) -- (8.5,7.2);
		\draw [orange,->,>=stealth] (8.5,6.8) -- (7.5,6.8);
		\draw [orange,->,>=stealth] (6.8,7.5) -- (6.8,8.5);
		\draw [orange,->,>=stealth] (7.2,8.5) -- (7.2,7.5);
		
		\draw [orange,->,>=stealth] (7.5,9) -- (8.5,9);
		\draw [orange,->,>=stealth] (6.8,9.5) -- (6.8,10.5);
		\draw [orange,->,>=stealth] (7.2,10.5) -- (7.2,9.5);
		
		\draw [orange,->,>=stealth] (7.5,11) -- (8.5,11);
		\draw [orange,->,>=stealth] (7,11.5) -- (7,12.5);
		
		\draw [orange,->,>=stealth] (9.5,3.2) -- (10.5,3.2);
		\draw [orange,->,>=stealth] (10.5,2.8) -- (9.5,2.8);
		\draw [orange,->,>=stealth] (8.8,3.5) -- (8.8,4.5);
		\draw [orange,->,>=stealth] (9.2,4.5) -- (9.2,3.5);
		
		\draw [orange,->,>=stealth] (9.5,5.2) -- (10.5,5.2);
		\draw [orange,->,>=stealth] (10.5,4.8) -- (9.5,4.8);
		\draw [orange,->,>=stealth] (8.8,5.5) -- (8.8,6.5);
		\draw [orange,->,>=stealth] (9.2,6.5) -- (9.2,5.5);
		
		\draw [orange,->,>=stealth] (9.5,7.2) -- (10.5,7.2);
		\draw [orange,->,>=stealth] (10.5,6.8) -- (9.5,6.8);
		\draw [orange,->,>=stealth] (9,7.5) -- (9,8.5);
		
		\draw [orange,->,>=stealth] (11.5,3) -- (12.5,3);
		\draw [orange,->,>=stealth] (10.8,3.5) -- (10.8,4.5);
		\draw [orange,->,>=stealth] (11.2,4.5) -- (11.2,3.5);
		
		\draw [orange,->,>=stealth] (11.5,5) -- (12.5,5);
		\draw [orange,->,>=stealth] (10.8,5.5) -- (10.8,6.5);
		\draw [orange,->,>=stealth] (11.2,6.5) -- (11.2,5.5);
		
		\draw [orange,->,>=stealth] (11.5,7) -- (12.5,7);
		\draw [orange,->,>=stealth] (10.8,7.5) -- (10.8,8.5);
		\draw [orange,->,>=stealth] (11.2,8.5) -- (11.2,7.5);
		
		\draw [orange,->,>=stealth] (11.5,9) -- (12.5,9);
		\draw [orange,->,>=stealth] (10.5,9) -- (9.5,9);
		\draw [orange,->,>=stealth] (10.8,9.5) -- (10.8,10.5);
		\draw [orange,->,>=stealth] (11.2,10.5) -- (11.2,9.5);
		
		\draw [orange,->,>=stealth] (11.5,11) -- (12.5,11);
		\draw [orange,->,>=stealth] (10.5,11) -- (9.5,11);
		\draw [orange,->,>=stealth] (11,11.5) -- (11,12.5);
		
		\node at (2.6,2) {\textcolor{red}{{\normalsize+0.010}}}; 

		\node at (6,1.2) {\textcolor{blue}{{\normalsize-0.010}}}; 
		
		\node at (2.6,4) {\textcolor{red}{{\normalsize+0.010}}}; 
		
		\node at (8,1.2) {\textcolor{blue}{{\normalsize-0.010}}}; 
		
		\node at (2.6,6) {\textcolor{red}{{\normalsize+0.010}}}; 
		
		\node at (10,1.2) {\textcolor{blue}{{\normalsize-0.010}}}; 
		
		\node at (4,7.2) {\textcolor{red}{{\normalsize+0.100}}}; 
		
		\node at (6,7.2) {\textcolor{red}{{\normalsize+0.123}}}; 
		
		\node at (1,12) {\textcolor{red}{{\normalsize+0.100}}}; 
		
		\node at (8,7.2) {\textcolor{red}{{\normalsize+0.123}}}; 
		
		\node at (10,7.2) {\textcolor{red}{{\normalsize+0.123}}}; 

		\node at (11,12) {\textcolor{red}{{\normalsize+0.424}}}; 
		
		\end{tikzpicture}
		\caption{Sparse reward modification for MDP experiment 3.}
		\label{fig:grid-world-2-complete-sparse}
\end{figure}

Finally, we show the result on attacking LQR with $\alpha=1$. The attack cost is $\|\bv r-\bv r^0\|_1=5.44$, compared to $\|\bv r^0\|_1=2088.57$.
In Figure~\ref{fig:poison-LQR-sparse}, we plot the clean and poisoned trajectory of the vehicle, together with the reward modification in each time step. 
The attack is as effective as with a dense 2-norm attack in Figure~\ref{fig:poison-LQR}.
However, the poisoning is highly sparse: only 10 out of 400 rewards are changed.
\begin{figure}[H]
	\begin{subfigure}[t]{.4\textwidth}
		\centering
		\includegraphics[width=0.95\textwidth, height=0.85\textwidth]{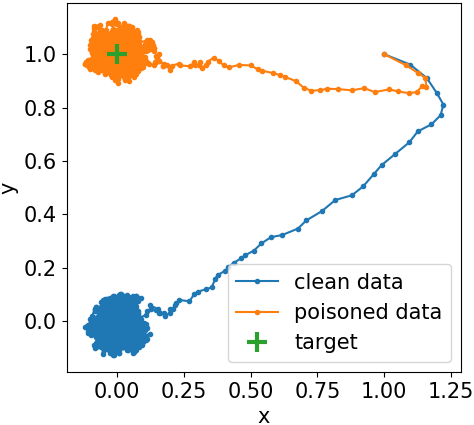}
		\caption{Clean and poisoned vehicle trajectory.}
		\label{fig:traj-LQR-sparse}
	\end{subfigure}%
	\begin{subfigure}[t]{.4\textwidth}
		\centering
		\includegraphics[width=0.95\textwidth, height=0.85\textwidth]{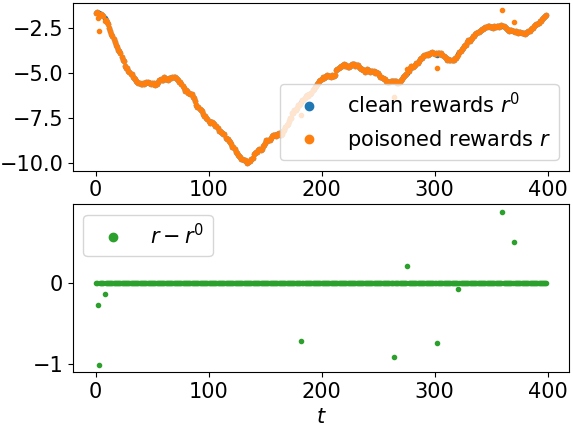}
		\caption{Clean and poisoned rewards.}
		\label{fig:poisoned_rewards-sparse}
	\end{subfigure}%
	\caption{Sparse-poisoning a vehicle running LQR in 4D state space.}
	\label{fig:poison-LQR-sparse}
\end{figure}

\section{Derivation of Discounted Discrete-time Algebraic Riccati Equation}
We provide a derivation for the discounted Discrete-time Algebraic Riccati Equation. For simplicity, we consider the noiseless case, but the derivation easily generalizes to noisy case. We consider the loss function is a general quadratic function w.r.t. $s$ as follows:
\begin{equation}
\begin{aligned}
L(s, a)  =\frac{1}{2} s^{\top} Q s + q^{\top} s + c + a^{\top} R a.
\end{aligned} 
\end{equation}
When $q = 0, c= 0$, we recover the classic LQR setting.
Assume the general value function takes the form $V(s) = \frac{1}{2} s^{\top} X s + s^{\top} x + v$. Let $Q(s,a)$ (note that this is different notation from the $Q$ matrix in $L(s,a)$) be the corresponding action value function. We perform dynamics programming as follows:
\begin{equation}
\begin{aligned}
Q(s,a) & = \frac{1}{2} s^{\top} Q s + q^{\top} s + c + a^{\top}R a + \gamma V(As + B a) \\
& = \frac{1}{2} s^{\top} Q s + q^{\top} s + c + a^{\top} R a + \gamma \left(\frac{1}{2} (As+Ba)^{\top} X (As+Ba) + (As+Ba)^{\top} x + v\right) \\
& = \frac{1}{2} s^{\top} (Q + \gamma A^{\top}X A )s + \frac{1}{2} a^{\top} (R+ \gamma B^{\top}X B) a  + s^{\top}( \gamma A^{\top}XB )a \\
& ~ + s^{\top}(q + \gamma A^{\top}x) + a^{\top}(\gamma B^{\top}x) + (c + \gamma v).
\end{aligned} 
\end{equation}
We minimize $a$ above:
\begin{equation}
\begin{aligned}
&(R+\gamma B^{\top}X B)a + \gamma B^{\top}X A s + \gamma B^{\top}x = 0 \\
& ~ \Rightarrow  a = -\gamma(R+\gamma B^{\top}X B)^{-1} B^{\top}X A s - \gamma(R+\gamma B^{\top}XB)^{-1} B^{\top}x \triangleq K s + k.
\end{aligned}
\end{equation} 
Now we substitute it back to $Q(s,a)$ and regroup terms, we get:
\begin{equation}
\begin{aligned}
V(s)  =& \frac{1}{2} s^{\top}( Q + \gamma A^{\top}X A + K^{\top}(R+\gamma B^{\top}XB) K  + 2\gamma A^{\top}XB K)s \\
&  +  s^{\top}(K^{\top}(R+\gamma B^{\top}XB)k + \gamma A^{\top}XB k + q + \gamma A^{\top}x + \gamma K^{\top}B^{\top}x) + C
\end{aligned}
\end{equation}
for some constant $C$, which gives us the following recursion:
\begin{equation}
\begin{aligned}
&X = \gamma A^{\top} X A - \gamma^2 A^{\top}XB(R+\gamma B^{\top}XB)^{-1}B^{\top}X A+Q,\\
& x =   q + \gamma  (A+BK)^{\top} x. 
\end{aligned}
\end{equation}

%% file: PPRL.bbl
\begin{thebibliography}{22}
\providecommand{\natexlab}[1]{#1}
\providecommand{\url}[1]{\texttt{#1}}
\expandafter\ifx\csname urlstyle\endcsname\relax
  \providecommand{\doi}[1]{doi: #1}\else
  \providecommand{\doi}{doi: \begingroup \urlstyle{rm}\Url}\fi

\bibitem[Asmuth et~al.(2008)Asmuth, Littman, and Zinkov]{asmuth2008potential}
John Asmuth, Michael~L Littman, and Robert Zinkov.
\newblock Potential-based shaping in model-based reinforcement learning.
\newblock In \emph{Proceedings of the 23rd national conference on Artificial
  intelligence-Volume 2}, pages 604--609. AAAI Press, 2008.

\bibitem[Biggio and Roli(2018)]{biggio2018wild}
Battista Biggio and Fabio Roli.
\newblock Wild patterns: Ten years after the rise of adversarial machine
  learning.
\newblock \emph{Pattern Recognition}, 84:\penalty0 317--331, 2018.

\bibitem[Biggio et~al.(2012)Biggio, Nelson, and Laskov]{biggio2012poisoning}
Battista Biggio, B~Nelson, and P~Laskov.
\newblock Poisoning attacks against support vector machines.
\newblock In \emph{29th International Conference on Machine Learning}, pages
  1807--1814. ArXiv e-prints, 2012.

\bibitem[Brown and Niekum(2019)]{brown2019machine}
Daniel~S Brown and Scott Niekum.
\newblock Machine teaching for inverse reinforcement learning: Algorithms and
  applications.
\newblock In \emph{Proceedings of the AAAI Conference on Artificial
  Intelligence}, volume~33, pages 7749--7758, 2019.

\bibitem[Cakmak and Lopes(2012)]{cakmak2012algorithmic}
Maya Cakmak and Manuel Lopes.
\newblock Algorithmic and human teaching of sequential decision tasks.
\newblock In \emph{Twenty-Sixth AAAI Conference on Artificial Intelligence},
  2012.

\bibitem[Dean et~al.(2017)Dean, Mania, Matni, Recht, and Tu]{dean2017sample}
Sarah Dean, Horia Mania, Nikolai Matni, Benjamin Recht, and Stephen Tu.
\newblock On the sample complexity of the linear quadratic regulator.
\newblock \emph{arXiv preprint arXiv:1710.01688}, 2017.

\bibitem[Devlin and Kudenko(2012)]{devlin2012dynamic}
Sam~Michael Devlin and Daniel Kudenko.
\newblock Dynamic potential-based reward shaping.
\newblock In \emph{Proceedings of the 11th International Conference on
  Autonomous Agents and Multiagent Systems}, pages 433--440. IFAAMAS, 2012.

\bibitem[Diamond and Boyd(2016)]{cvxpy}
Steven Diamond and Stephen Boyd.
\newblock {CVXPY}: A {P}ython-embedded modeling language for convex
  optimization.
\newblock \emph{Journal of Machine Learning Research}, 17\penalty0
  (83):\penalty0 1--5, 2016.

\bibitem[Huang et~al.(2011)Huang, Joseph, Nelson, Rubinstein, and
  Tygar]{huang2011adversarial}
Ling Huang, Anthony~D Joseph, Blaine Nelson, Benjamin~IP Rubinstein, and
  JD~Tygar.
\newblock Adversarial machine learning.
\newblock In \emph{Proceedings of the 4th ACM workshop on Security and
  artificial intelligence}, pages 43--58. ACM, 2011.

\bibitem[Huang et~al.(2017)Huang, Papernot, Goodfellow, Duan, and
  Abbeel]{huang2017adversarial}
Sandy Huang, Nicolas Papernot, Ian Goodfellow, Yan Duan, and Pieter Abbeel.
\newblock Adversarial attacks on neural network policies.
\newblock \emph{arXiv preprint arXiv:1702.02284}, 2017.

\bibitem[Jun et~al.(2018)Jun, Li, Ma, and Zhu]{jun2018adversarial}
Kwang-Sung Jun, Lihong Li, Yuzhe Ma, and Jerry Zhu.
\newblock Adversarial attacks on stochastic bandits.
\newblock In \emph{Advances in Neural Information Processing Systems}, pages
  3640--3649, 2018.

\bibitem[Kamalaruban et~al.(2019)Kamalaruban, Devidze, Cevher, and
  Singla]{kamalaruban19}
P~Kamalaruban, R~Devidze, V~Cevher, and A~Singla.
\newblock Interactive teaching algorithms for inverse reinforcement learning.
\newblock In \emph{28th International Joint Conference on Artificial
  Intelligence}, pages 604--609, 2019.

\bibitem[Koh et~al.(2018)Koh, Steinhardt, and Liang]{koh2018stronger}
Pang~Wei Koh, Jacob Steinhardt, and Percy Liang.
\newblock Stronger data poisoning attacks break data sanitization defenses.
\newblock \emph{arXiv preprint arXiv:1811.00741}, 2018.

\bibitem[Li et~al.(2016)Li, Wang, Singh, and Vorobeychik]{li2016data}
Bo~Li, Yining Wang, Aarti Singh, and Yevgeniy Vorobeychik.
\newblock Data poisoning attacks on factorization-based collaborative
  filtering.
\newblock In \emph{Advances in neural information processing systems}, pages
  1885--1893, 2016.

\bibitem[Liu and Shroff(2019)]{liu2019data}
Fang Liu and Ness Shroff.
\newblock Data poisoning attacks on stochastic bandits.
\newblock In \emph{International Conference on Machine Learning}, pages
  4042--4050, 2019.

\bibitem[Ma et~al.(2018)Ma, Jun, Li, and Zhu]{ma2018data}
Yuzhe Ma, Kwang-Sung Jun, Lihong Li, and Xiaojin Zhu.
\newblock Data poisoning attacks in contextual bandits.
\newblock In \emph{International Conference on Decision and Game Theory for
  Security}, pages 186--204. Springer, 2018.

\bibitem[Mei and Zhu(2015)]{mei2015using}
Shike Mei and Xiaojin Zhu.
\newblock Using machine teaching to identify optimal training-set attacks on
  machine learners.
\newblock In \emph{Twenty-Ninth AAAI Conference on Artificial Intelligence},
  2015.

\bibitem[Ng et~al.(1999)Ng, Harada, and Russell]{ng1999policy}
Andrew~Y Ng, Daishi Harada, and Stuart Russell.
\newblock Policy invariance under reward transformations: Theory and
  application to reward shaping.
\newblock In \emph{ICML}, volume~99, pages 278--287, 1999.

\bibitem[Wang and Chaudhuri(2018)]{wang2018data}
Yizhen Wang and Kamalika Chaudhuri.
\newblock Data poisoning attacks against online learning.
\newblock \emph{arXiv preprint arXiv:1808.08994}, 2018.

\bibitem[Wiewiora(2003)]{wiewiora2003potential}
Eric Wiewiora.
\newblock Potential-based shaping and q-value initialization are equivalent.
\newblock \emph{Journal of Artificial Intelligence Research}, 19:\penalty0
  205--208, 2003.

\bibitem[Xiao et~al.(2015)Xiao, Biggio, Brown, Fumera, Eckert, and
  Roli]{xiao2015feature}
Huang Xiao, Battista Biggio, Gavin Brown, Giorgio Fumera, Claudia Eckert, and
  Fabio Roli.
\newblock Is feature selection secure against training data poisoning?
\newblock In \emph{International Conference on Machine Learning}, pages
  1689--1698, 2015.

\bibitem[Zhang and Zhu(2019)]{zhang2019online}
Xuezhou Zhang and Xiaojin Zhu.
\newblock Online data poisoning attack.
\newblock \emph{arXiv preprint arXiv:1903.01666}, 2019.

\end{thebibliography}
